\documentclass[12pt]{article}

\usepackage{subfigure}

\usepackage{latexsym}
\usepackage{graphics}
\usepackage{amsmath}
\usepackage{xspace}
\usepackage{amssymb}
\usepackage{psfrag}
\usepackage{epsfig}
\usepackage{amsthm}
\usepackage{pst-all}
\usepackage{hyperref}

\usepackage{color}

\definecolor{Red}{rgb}{1,0,0}
\definecolor{Blue}{rgb}{0,0,1}
\definecolor{Olive}{rgb}{0.41,0.55,0.13}
\definecolor{Yarok}{rgb}{0,0.5,0}
\definecolor{Green}{rgb}{0,1,0}
\definecolor{MGreen}{rgb}{0,0.8,0}
\definecolor{DGreen}{rgb}{0,0.55,0}
\definecolor{Yellow}{rgb}{1,1,0}
\definecolor{Cyan}{rgb}{0,1,1}
\definecolor{Magenta}{rgb}{1,0,1}
\definecolor{Orange}{rgb}{1,.5,0}
\definecolor{Violet}{rgb}{.5,0,.5}
\definecolor{Purple}{rgb}{.75,0,.25}
\definecolor{Brown}{rgb}{.75,.5,.25}
\definecolor{Grey}{rgb}{.5,.5,.5}

\newcommand{\pr}{\mathbb{P}}
\newcommand{\E}[1]{\mathbb{E}\!\left[#1\right]}
\newcommand{\R}{\mathbb{R}}

\newcommand{\ignore}[1]{\relax}

\newtheorem{theorem}{Theorem}[section]

\newtheorem{lemma}[theorem]{Lemma}

\newtheorem{proposition}[theorem]{Proposition}
\newtheorem{coro}[theorem]{Corollary}

\definecolor{Red}{rgb}{1,0,0}
\definecolor{Blue}{rgb}{0,0,1}
\definecolor{Olive}{rgb}{0.41,0.55,0.13}
\definecolor{Green}{rgb}{0,1,0}
\definecolor{MGreen}{rgb}{0,0.8,0}
\definecolor{DGreen}{rgb}{0,0.55,0}
\definecolor{Yellow}{rgb}{1,1,0}
\definecolor{Cyan}{rgb}{0,1,1}
\definecolor{Magenta}{rgb}{1,0,1}
\definecolor{Orange}{rgb}{1,.5,0}
\definecolor{Violet}{rgb}{.5,0,.5}
\definecolor{Purple}{rgb}{.75,0,.25}
\definecolor{Brown}{rgb}{.75,.5,.25}
\definecolor{Grey}{rgb}{.5,.5,.5}
\definecolor{Pink}{rgb}{1,0,1}
\definecolor{DBrown}{rgb}{.5,.34,.16}
\definecolor{Black}{rgb}{0,0,0}

\usepackage{float}

\author{
{\sf David Gamarnik}\thanks{MIT; e-mail: {\tt gamarnik@mit.edu}. Research supported  by the NSF grants CMMI-1335155.}
\and
{\sf Ilias Zadik}\thanks{MIT; e-mail: {\tt izadik@mit.edu}}
}

\begin{document}

\title{High-Dimensional Regression with Binary Coefficients. Estimating Squared Error and the Phase Transition}
\date{\today}

\maketitle

\begin{abstract}
 We consider a sparse linear regression model $Y=X\beta^{*}+W$ where $X$ is $n\times p$ matrix Gaussian i.i.d.  entries, $W$
is $n\times 1$ noise vector with i.i.d. mean zero Gaussian entries and standard deviation $\sigma$, 
and $\beta^{*}$ is $p\times 1$ binary vector with support size (sparsity)  $k$.
Using a  novel conditional second moment method  we obtain a tight up to a 
multiplicative constant approximation of the optimal squared error $\min_{\beta}\|Y-X\beta\|_{2}$,  
where the minimization is over all $k$-sparse binary
vectors $\beta$.
The approximation reveals interesting structural properties of the 
underlying regression problem. In particular,

\begin{enumerate}
\item [(a)] We establish that $n^{*}=2k\log p/\log (2k/\sigma^{2}+1)$ is a phase transition point
with the following ``all-or-nothing'' property.
When $n$ exceeds $n^{*}$,  $(2k)^{-1}\|\beta_{2}-\beta^*\|_0\approx 0$, and when $n$ is  below $n^{*}$, 
$(2k)^{-1}\|\beta_{2}-\beta^*\|_0\approx 1$, where $\beta_2$ is the optimal solution achieving the smallest squared error.
With this we prove that  $n^{*}$ is the asymptotic threshold for recovering $\beta^*$  information theoretically. Note that $n^*$ is asymptotically 
 below the threshold  $n_{\text{LASSO/CS}}=(2k+\sigma^2)\log p$, above which 
the LASSO and Compressive Sensing methods are able to recover $\beta^*$.

\item [(b)] We compute the squared error for an intermediate problem $\min_{\beta}\|Y-X\beta\|_{2}$ where minimization is restricted to vectors
$\beta$ with $\|\beta-\beta^{*}\|_0=2k \zeta$, for some fixed ratio $\zeta\in [0,1]$. 
We show that a lower bound part $\Gamma(\zeta)$ of the estimate, which essentially corresponds to the estimate based on the first moment method,
undergoes a phase transition  at three different thresholds, namely 
$n_{\text{inf,1}}=\sigma^2\log p$, which is information theoretic bound for recovering $\beta^*$ when $k=1$
and $\sigma$ is large, then at $n^{*}$ and finally at $n_{\text{LASSO/CS}}$. 

\item [(c)] We establish a certain Overlap Gap Property (OGP) on the space of all binary vectors $\beta$ when $n\le ck\log p$ for sufficiently small constant $c$.
By drawing a connection with a similar OGP exhibited by many randomly generated constraint satisfaction problems and statistical physics models, 
we conjecture that OGP is the source of algorithmic
hardness of solving the minimization problem $\min_{\beta}\|Y-X\beta\|_{2}$ in the regime $n<n_{\text{LASSO/CS}}$.

\end{enumerate}

\end{abstract}


\section{Introduction}\label{section:Intro}

In this paper we study the linear regression model $$Y=X\beta^{*}+W$$where $X$ is a data $n\times p$ matrix, $W$
is a $n\times 1$ noise vector, and $\beta^{*}$ is the (unknown) $p\times 1$ vector of regression coefficients. We refer to $n$ as the number of samples and $p$ as the number of features for the model. The goal is to recover $\beta^*$ from having access only to the data matrix $X$ and the noisy linear observations $Y$. Throughout the paper we focus on the case of stochastic error measurement noise where each $W_i$ is an i.i.d. sample from a $N\left(0,\sigma^2\right)$ for some parameter $\sigma>0$.

This work studies the high dimensional regime where $n \ll p$ and $p \rightarrow + \infty$. High-dimensionality is motivated by various statistical applications over the last decade for example in the field of radiology and biomedical imaging (see e.g. \cite{DonohoMRI} and references therein) and in the field of genomics  \cite{BickelGenome}, \cite{JASAgenomics}. Note that this is an, in principle, impossible regime for (exact) inference of $\beta^*$ from $(Y,X)$  ; the underlying linear system, even at the extreme case $\sigma=0$, is underdetermined. For this reason, following a large line of research, we study the linear model under the additional structural assumption that the vector of coefficients is $k$-sparse, that is the support size of $\beta^*$ (i.e. the number of regression coefficients with non-zero value)  equals to some positive integer parameter $k$ which is usually taken much smaller than $p$. Sparsity is a well-established assumption in the statistics literature, with various applications for example in compressed sensing~\cite{candes2005decoding}, \cite{donoho2006compressed} , biomedical imaging \cite{MRImed}, \cite{DonohoMRI}  and sensor networks \cite{CSwireless}, \cite{CSsensor}.

 In this paper we focus on the fundamental statistical task of recovering the support of $\beta^*$   \cite{Cai2012}, \cite{Martin11}, that is inferring from $(Y,X)$ the location of its non-zero coordinates of $\beta^*$. The support recovery problem has attracted a lot of attention in  recent years, because it naturally arises in many contexts including gene selection in genomics \cite{Ho08}, \cite{Hu10}, \cite{Hu09} and radar signal processing \cite{Du17}, \cite{ZH01}, \cite{Chen99}. It is worth mentioning that support recovery is also known in the literature as sparsity pattern recovery task \cite{Galen13}, variable selection (see \cite{Ed2012} and references therein) or model selection \cite{Model93}, \cite{Aos06}.


We investigate the fundamental statistical and algorithmic limits of the high dimensional linear regression setting.  Specifically we are interested in answering the following questions,
\begin{center}
\textit{For which values of $n$ is it information-theoretic possible to infer the support of $\beta^*$? \\
When can this inference task be made in a computationally efficient way?}
\end{center} As it is well-known the study of high dimensional linear regression poses multiple computational and statistical challenges leading to a vast research literature on the topic (see e.g.  \cite{Fan10}, \cite{Book15} and references therein). For this reason, at the goal of establishing tight answers to the above questions, we study the problem under additional assumptions on the data matrix $X$ and the vector $\beta^*$. We expect our result to provide intuition for general setting of the high dimensional linear regression model under the sparsity constraint. First we assume that each row of $X$ is generated as an iid sample from an isotropic $\mathcal{N}\left(0,\Sigma\right),$ where we take $\Sigma=I_{p}$. Note that the Gaussianity of the data rows is, in a standard way, justified from the Central Limit theorem and is very common in the literature \cite{Casto11}, \cite{Lucas17}, \cite{Van13}, \cite{Tony18}, \cite{wainwright2009sharp}, \cite{wainwright2009information},\cite{wang2010information}. Furthermore, the case $\Sigma=I_p$, which can be considered unrealistic from an applied point of view, has been considered broadly in the literature as an idealized assumption which allows broader technical development which can usually be generalized  \cite{Casto11}, \cite{Lucas17}, \cite{wainwright2009sharp}, \cite{wainwright2009information}, \cite{wang2010information}. A famous line of research, also relevant to this results of this paper, where this assumptions has been very helpful is the analysis of the LASSO optimization problem (see Chapter 11 in \cite{Book15} and references therein). Second, we assume that the non-zero regression coefficients $\beta^*_i$ are all equal with each other and (after rescaling) equal to one; that is we assume we assume a binary $\beta^* \in \{0,1\}^p$. Despite our technical motivation for focusing on the binary case, the case of binary and more generally discrete-valued $\beta^*$ has received a large interest in the study of wireless communications and information-theory literature \cite{Hassibi98}, \cite{Hassibi02}, \cite{Brunel99}, \cite{GZ18}, \cite{Thr19}, \cite{Zad19}.  All the earlier results in the literature discussed below are adopted to these assumptions.

A lot of work has been devoted in particular to finding computationally efficient ways for recovering the support of $\beta^*$. 
In the noiseless setting ($W=0$), Donoho and Tanner show in~\cite{donoho2006counting} that the simple linear program: 
$\min ||\beta||_1$ subject to $Y=X\beta$,  
will have with high probability (w.h.p.) $\beta^*$ as its optimal solution if 
$n \geq 2\left(1+\epsilon\right)k\log p$. Here and below $\|\cdot\|_1$ and $\|\cdot\|_2$
denote the standard $\ell_1$ and $\ell_2$ norms, respectively: $\|x\|_1=\sum_{1\le i\le p} |x_i|$ and 
$\|x\|_2=\left(\sum_{1\le i\le p} x_i^2\right)^{1\over 2}$
for every $x\in\R^p$.
In the noisy setting, sufficient and necessary conditions have been found so that the $\ell_{1}$- constrained quadratic programming, also known as 
LASSO: $\min_{\beta \in \mathbb{R}^p} \{ ||Y-X\beta||_2^2+\lambda_p||\beta||_1 \},$ for appropriately chosen $\lambda_p>0$, recovers the correct support
of $\beta^{*}$,
\cite{graph},\cite{wainwright2009sharp},\cite{Zhao}. See also the recent book~\cite{foucart2013mathematical}.
In particular, Wainwright~\cite{wainwright2009sharp} showed that if $X$ is a Gaussian random matrix 
and $W$ is a Gaussian noise vector with variance $\sigma^2$ such that $\frac{\sigma^2}{k} \rightarrow 0$, 
then for every arbitrarily small constant $\epsilon>0$ and for $n>\left(1+\epsilon\right)(2k+\sigma^2)\log p$,
the LASSO based method recovers the support of $\beta^{*}$ exactly w.h.p. At the same time
given any $\epsilon>0$, if $n<\left(1-\epsilon\right)(2k+\sigma^2)\log p$, then the LASSO based method provably fails to recover the support of $\beta^*$ exactly, 
also w.h.p. We note that the impact of $\sigma^2$ on this threshold is asymptotically negligible when $\sigma^2/k\rightarrow 0$. 
It will be convenient for us to keep it though and thus we denote $(2k+\sigma^2)\log p$ by $n_{\text{LASSO/CS}}$. 
At the present
time no tractable (polynomial time) algorithms are known for the support recovery when $n\le n_{\rm LASSO/CS}$.

On the complimentary direction, results regarding the information theoretic limits for the problem of support 
recovery have also been obtained~\cite{donoho2006counting},\cite{wainwright2009information},\cite{wang2010information}, \cite{Galen12}, \cite{Reeves13}, \cite{Scarlett15}.
These papers are devoted to obtaining bounds on the minimum
sampling size $n$ so that the support recovery problem is solvable by any algorithmic methods, regardless of the algorithmic complexity,
including for example the brute force method of exhaustive search. An easy corollary of Theorem~2 in~\cite{wainwright2009information}, which follows from an appropriate use of Fano's inequality,
when applied to our context below involving vectors $\beta^{*}$ with binary values, yields one information-theoretic lower bound. it is shown that if $n<\left(1-\epsilon\right) \sigma^2 \log p$, 
then for every support recovery algorithm, a binary vector $\beta^*$ can be constructed in such a way 
that the underlying algorithm  fails to recover $\beta^*$ exactly, with probability at least $\frac{\epsilon}{2}$. Interestingly, this lower bound value does not depend on the value of $k$. Viewing the problem from the Gaussian channel 
perspective, vector $Y$ can be viewed as a noisy encoding of $\beta^*$ through the code book $X$ and in our case
the sparsity $k$ becomes the strength of this Gaussian channel. Using the tight characterization of the Gaussian communication channel capacity (see e.g. Theorem 10.1.1. in \cite{Cover}) when $k=1$, the information theoretic limit of recovering the unit
bit support of $\beta^*$ is $\log p/\log(1+1/\sigma^2)$ which is $\sigma^2\log p$ asymptotically when $\sigma$ is large.
We let  $n_{\rm inf,1}\triangleq \sigma^2\log p$. Subsequently, it was shown by Wang et al~\cite{wang2010information} using similar ideas that the exact recovery of $\beta^*$ is information theoretically impossible when $n$ smaller than $n^*\triangleq 2k\log p/\log(1+2k/\sigma^2)$,
where $n^*$ is the information theoretic limit of this Gaussian channel for general $k$.
The critical threshold $n^*$ will play a fundamental role in our paper. We note that the result above does not preclude the possibility
of the existence of an algorithm which recovers some portion of the support of $\beta^*$ and this question is one of the motivation for the present work.

The regime $n\in [n_{\rm inf,1}, n_{\rm LASSO/CS}]$ remains largely unexplored from the algorithmic perspective,
 and the present paper is devoted to studying this regime.
Towards this goal, for the regression model  $Y=X\beta^{*}+W$, we consider the corresponding maximum likelihood estimation problem: 
\begin{align*}
\begin{array}{clc} \left(\Phi_2\right) & \min  &n^{-\frac{1}{2}}\|Y-X\beta\|_{2} \\ &\text{s.t.}&\beta \in \{0,1\}^p  \\
&& \|\beta\|_0=k,
\end{array}
\end{align*}
where $\|\beta\|_0$ is the sparsity of $\beta$. Namely, it is the cardinality of the set $\{ i \in [p] \big{|} \beta_i \not = 0 \}$.
We denote by $\phi_2$ its optimal value and by $\beta_{2}$ the unique optimal solution. 
As above, the matrix $X$ is assumed to have i.i.d. standard normal entries,  the elements of the noise
vector $W$ are assumed to have i.i.d. zero mean normal entries with variance $\sigma^{2}$, and the vector $\beta^{*}$ is assumed 
to be binary $k$-sparse; $\|\beta^*\|_0=k$.
In particular, we assume that the sparsity $k$ is known to the optimizer.
The normality of the entries of $X$ is not an essential assumption for our results, since the Central Limit Theorem based estimates
can be easily used instead. We adopt however the normality assumption for simplicity. The normality of the entries of $W$
is more crucial, since our large deviation estimates arising in the application of the conditional second moment depend on this 
assumption. It is entirely possible though that similar results are derivable by applying the large deviations estimates for the 
underlying distribution of entries of $Y$ in the general case.

We address two questions in this paper: (a) What is the  value of the  squared error  estimator 
$\min_{\beta \in \{0,1\}^p,\|\beta\|_0=k}\|Y-X\beta\|_{2}=\|Y-X\beta_2\|_2$; and (b) how well does the optimal vector $\beta_{2}$  approximate  the 
ground truth vector $\beta^{*}$?

Our problem setup, including the assumption that $\beta^{*}$ is binary, has an important theoretical motivation. The gap between the information theoretic and algorithmic
bounds is particularly profound when $\beta^{*}$ is binary. Observe, for example, that in the noiseless setting ($W=0$) even one sample ($n=1$)
is sufficient to recover $\beta^{*}$ by brute force search, whereas $n_{\text{LASSO/CS}}=(2k+\sigma^2)\log p$. The optimization problem $\Phi_{2}$ is naturally hard
algorithmically since it involves a combinatorial constraint $\|\beta\|_{0}=k$. At the same time, it can be cast as an integer programming
optimization problem, and the advances in this area make such problems solvable in many practical settings (see ~\cite{BertsimasRegression} and references therein). Thus
the performance of the optimization problem $\Phi_2$ is still of interest, even though formally, it is not proven to be a tractable algorithmic problem.
Note though that the 
algorithmic hardness of solving the minimization problem subject to the constraint on $\|\beta\|_0$
pertains to the worst case instances and does not apply to settings involving randomly generated data such as $X$ and $Y$.
In fact, one of the goals of this paper is to shed some light on possible sources of the apparent 
algorithmic hardness of this problem in the case when $X$ and $Y$ are indeed random.

\subsection*{Results}
Towards the goals outlined above we obtain several structural results regarding the optimization problem $\Phi_{2}$, its optimal value $\phi_2$, 
and its optimal
solution $\beta_{2}$. We introduce a new method of analysis based on a certain conditional second moment method. The method
will be explained below in high level terms.  Using this method we obtain a tight up to a 
multiplicative constant approximation of the squared error $\phi_2$ w.h.p., as parameters $p,n,k$ diverge to infinity,
and $n\le ck\log p$ for a small constant $c$. Some additional assumptions
on $p,n$ and $k$ are needed and will be introduced in the statements of the results.
The approximation enables us to  reveal interesting structural properties of the 
underlying optimization problem $\Phi_{2}$. In particular,

\begin{enumerate}
\item [(a)] We prove that  $n^{*}=2k\log p/\log (2k/\sigma^{2}+1)$ which was shown in~\cite{wang2010information}
to be the information theoretic lower bound 
for the exact recovery of $\beta^*$
is the phase transition point
with the following ''all-or-nothing'' property.
When $n$ exceeds $n^{*}$ asymptotically,  $(2k)^{-1}\|\beta_{2}-\beta^*\|_0\approx 0$, and when $n$ is asymptotically below $n^{*}$, 
$(2k)^{-1}\|\beta_{2}-\beta^*\|_0\approx 1$. Namely, when $n>n^{*}$ the recovery of $\beta^{*}$ is achievable via solving $\Phi_{2}$, 
whereas below $n^{*}$ the optimization problem $\Phi_{2}$ ``misses'' the ground truth vector $\beta^{*}$ almost entirely.
Since, as discussed above, when $n<n^*$, the recovery of $\beta^*$ is impossible information theoretically, our result implies
that $n^*$  is indeed the information theoretic threshold for this problem.
We recall that  $n^{*}$
exceeds asymptotically the asymptotic one-bit ($k=1$) information theoretic threshold $n_{{\rm inf,1}}=\sigma^{2}\log p$, and is asymptotically below 
the LASSO/Compressive Sensing threshold 
$n_{\text{LASSO/CS}}=(2k+\sigma^2)\log p$. We note also that our result improves upon the result of Wainwright~\cite{wainwright2009information},
who shows that the recovery of $\beta^*$ is possible by the brute force search method, though only when $n$ is of the order $O(k\log p)$.

\item [(b)] We consider an intermediate optimization problem $\min_{\beta}\|Y-X\beta\|_{2}$ when the minimization is restricted to vectors
$\beta$ with $\|\beta-\beta^{*}\|_0=2k \zeta$, for some fixed ratio $\zeta\in [0,1]$. This is done towards deeper understanding of the problem 
$\Phi_{2}$.  We show that the function 
\begin{align*}
\Gamma(\zeta)\triangleq
\left(2\zeta k+\sigma^2\right)^{1\over 2}\exp\left(-\frac{\zeta k \log p}{n}\right),
\end{align*}
is, up to a multiplicative constant, a lower bound on this restricted optimization problem, and in the special case of $\zeta=0$ and $\zeta=1$,
it is also an upper bound, up to a multiplicative constant. Since $\Gamma$ is a log-concave function in $\zeta$,
returning to part (a) above, this implies that  that the squared error of the 
original optimization problem $\Phi_{2}$ is w.h.p. $\Gamma(0)=\sigma$ when $n>n^{*}$, and  is w.h.p.
$\Gamma(1)=\left(2 k+\sigma^2\right)^{1\over 2}\exp\left(-\frac{k \log p}{n}\right)$ when $n<n^{*}$, both up to multiplicative constants.
We further establish that the function $\Gamma$ exhibits phase transition property at all three important 
thresholds $n_{{\rm inf,1}}, n^{*}$ and $n_{\text{LASSO/CS}}$, described pictorially on Figures~\ref{fig:subfigures} in 
the next section.
In particular, we prove that when $n>n_{\text{LASSO/CS}}$,
$\Gamma(\zeta)$ is a strictly increasing function with minimum at $\zeta=0$, and when $n<n_{{\rm inf,1}}$, it is a strictly
decreasing function with minimum at $\zeta=1$. When $n^{*}<n<n_{\text{LASSO/CS}}$, $\Gamma(\zeta)$ is non-monotonic and achieves the
minimum value at $\zeta=0$, and when $n_{{\rm inf,1}}<n<n^{*}$, $\Gamma(\zeta)$ is again non-monotonic and achieves the
minimum value at $\zeta=1$. In the critical case $n=n^{*}$, both $\zeta=0$ and $\zeta=1$ are minimum values of $\gamma$. 

The results above suggest the following, albeit completely intuitive and heuristic picture, which 
is based on assuming that the function $\Gamma$ provides an accurate approximation of the value of $\phi_2$. 
When $n>n_{\rm LASSO/CS}$, a closer overlap with the ground truth vector $\beta^{*}$ allows for lower squared error value 
($\Gamma$ is increasing in $\zeta$). 
In this case the convex relaxation based methods such as LASSO and Compressive Sensing succeed in identifying $\beta^{*}$. We
conjecture that in this case even more straightforward, greedy type algorithms based on one step improvements might be able to 
recover $\beta^{*}$. At this stage, this remains a conjecture.

When $n$ is below $n_{\text{LASSO/CS}}$ but above $n^{*}$, the optimal solution $\beta_{2}$ of $\Phi_{2}$ 
still approximately coincides with $\beta^{*}$, but in this
case there is a proliferation of solutions which, while they 
achieve a sufficiently low squared error value, at the same time have  very little overlap with $\beta^{*}$. Considering
a cost value below the largest value of the function $\Gamma$, we obtain two groups of solutions: those with a ``substantial'' overlap with $\beta^{*}$
and those with a ``small'' even zero overlap with $\beta^{*}$. This motivates looking at the so-called Overlap Gap Property discussed in (c) below.

When $n$ is below $n^{*}$, there are solutions, and in particular the optimal solution $\beta_{2}$, which achieve  better squared error
value than
even the ground truth $\beta^{*}$. This is exhibited by the fact that the minimum value of $\Gamma$ is achieved at $\zeta=1$.
We are dealing here with the case of overfitting. While, information theoretically it is impossible to precisely recover $\beta^*$  in this regime, 
it is not clear whether in this case there exists any algorithm which can recover at least a portion of the support of $\beta^{*}$, 
algorithmic complexity aside.  We leave it as an interesting open question. 

When $n$ is below the ($k=1,$ large $\sigma$) information theoretic lower bound $n_{{\rm inf,1}}$, the overfitting situation is even more profound. 
Moving further away from $\beta^{*}$
allows for better and better squared error values ($\Gamma$ is decreasing in $\zeta$).

\item [(c)] Motivated by the results in the theory of spin glasses and the later results in the context of randomly generated
constraint satisfaction problems, and in light of  the evidence of the Overlap Gap Property (OGP) discussed above,
we consider the solution space geometry of the problem $\Phi_{2}$ as well as the restricted 
problem corresponding to the constraint $\|\beta-\beta^{*}\|_0=2\zeta k$. For many examples of randomly generated constraint 
satisfaction problems such as random K-SAT, proper coloring of a sparse random graph, the problem of finding a largest
independent subset of a sparse random graph, and many others, it has been conjectured and later established rigorously that solutions
achieving near optimality, or  solution satisfying a set of randomly generated constraints, break down into clusters separated
by cost barriers of a substantial size in some appropriate sense,~\cite{AchlioptasCojaOghlanRicciTersenghi},\cite{achlioptas2008algorithmic},
~\cite{montanari2011reconstruction},\cite{coja2011independent},\cite{gamarnik2014limits},\cite{rahman2014local},
\cite{gamarnik2014performance}. As a result, these models indeed exhibit the OGP. For example, independent
sets achieving near optimality in sparse random graph exhibit the OGP in the following sense. The intersection of every two such 
independent sets is either at most some value $\tau_{1}$ or at least some value $\tau_{2}>\tau_{1}$. This and similar properties were used
in~\cite{gamarnik2014limits},\cite{rahman2014local} and \cite{gamarnik2014performance} to establish a fundamental barriers on the power of so-called local algorithms
for finding nearly largest independent sets. The OGP was later established in a setting other than constraint satisfaction problems
on graphs, specifically in the context of finding a densest submatrix of a matrix with i.i.d. Gaussian entries~\cite{gamarnik2016finding}.

The non-monotonicity of the function $\Gamma$ for $n<n_{\text{LASSO/CS}}$ already suggests the presence of the OGP. Note that
for any value $r$ strictly below the maximum value $\max_{\zeta\in (0,1)}\Gamma(\zeta)$ we obtain the existence of two values
$\zeta_{1}<\zeta_{2}$, such  that for every $\zeta$ with $\Gamma(\zeta)\le r$, either $\zeta\le \zeta_{1}$ or $\zeta\ge \zeta_{2}$.
Namely, this property suggests that every binary vector achieving a cost at most $r$ either has the overlap at most $\zeta_{1}k$ with
$\beta^{*}$, or the  overlap at least $\zeta_{2}k$ with $\beta^{*}$. Unfortunately, this is no more than a guess, since $\Gamma(\zeta)$
provides only a lower bound on the optimization cost. Nevertheless, we establish that the OGP provably takes place w.h.p. when 
$C\sigma^{2}\log p\le n\le ck\log p$, for appropriately large constant $C$ and appropriately small constant $c$. Our result
takes advantage of the tight up to a multiplicative error estimates of the squared errors associated with the restricted 
optimization problem $\Phi_2$ with the restricted $\|\beta-\beta^*\|=2k\zeta$, discussed earlier.
It remains an intriguing open question to verify whether the optimization problem $\Phi_{2}$ is indeed algorithmically intractable
in this regime. 
\end{enumerate}

\subsection{On the Gaussian Assumptions on $X,W$.}

The high dimensional linear regression model on which the above results are obtained, is based on the idealized assumptions that $X \in \mathbb{R}^{n \times p}$ has iid rows drawn from a $\mathcal{N}\left(0,I_p\right)$ and $W \in \mathbb{R}^n$ has iid $\mathcal{N}\left(0,\sigma^2\right)$ entries. Naturally, the question is whether our structural results generalize beyond the present setting. 

Our results are expected to generalize much beyond these assumptions. For the case of $X$ and the Gaussian assumption $\mathcal{N}\left(0,I_p\right)$ on the rows of $X$, we expect similar results with respect to both the Gaussianity assumption and to the assumption of independence between the entries of each row. The reason is that the main use of the assumption, comes from estimating probabilities of the form $p_{t,y}:=\mathbb{P} \left(|\frac{X_{i1}+X_{i2}+\ldots +X_{ik}}{\sqrt{k}}-y| \leq t\right)$, over the randomness of $X$, for fixed values of $y \in \mathbb{R}, t>0,$ and an arbitrary row index $i \in \{1,2,\ldots,n\}$. Under our assumption of $X$, $\frac{X_{i1}+X_{i2}+\ldots +X_{ik}}{\sqrt{k}} $ follows a $\mathcal{N}\left(0,1\right)$ itself, allowing us for a standard estimation of $p_{t,y}$ by approximating the cumulative density function of the standard normal distribution. Note that under Gaussianity but with non-identity covariance matrix $\Sigma$, similar estimates can be easily established as $\frac{X_{i1}+X_{i2}+\ldots +X_{ik}}{\sqrt{k}} $ now follows $\mathcal{N}\left(0,\frac{1}{k}v_k^t \Sigma v_k\right)$ for $v_k \in \{0,1\}^p$ equal to 1 in the first $k$ coordinates and $0$ otherwise. Furthermore, if $X$ has non-Gaussian rows but weakly dependent entries per-row, since we consider the regime where $k \rightarrow + \infty$, similar estimates can be established using standard variants of Central Limit Theorem.

For the case of $W$ we expect our results to be more tied to assumptions following a similar to Gaussian behavior. Despite that, a generalization to any iid subgaussian noise distribution is expected. The reason is that the main probabilistic property used for each noise entry $W_i, i \in \{1,2,\ldots,n\}$  is that squared value $W_i^2$ has a finite moment generating function at some positive $\theta>0$, i.e. $\mathbb{E}\left[e^{\theta W_i^2}\right]< \infty$.

Finally, we would like to mention that one of the main results in this work, the presence of the Overlap Gap Property is \textit{a negative result}, as it presents a conjectured algorithmic barrier for the recovery problem. For such a result, establishing it in a specialized setting admittedly suffices to make a general hardness claim; the problem can only be at least as hard in a more general case.

\subsection{Methods} \label{Meth}
In order to obtain estimates of the squared error for the problem $\Phi_{2}$ we use a first and second moment method, which
we now describe in high level terms.
We begin with the following model which we call Pure Noise model, in which it is  assumed that $\beta^{*}=0$ and thus $Y$ is simply a vector of i.i.d.
zero mean  Gaussian random variables with variance $\sigma^{2}$. In this model the interest is on estimating the quantity $\min_{\beta} \|Y-X\beta\|_2$ where $\beta$ binary and $k$-sparse.

For every value $t>0$ we consider the counting random variable $Z_{t}$ equal to the
number of $k$-sparse binary $\beta$ such that $\|Y-X\beta\|_{\infty}\le t$, where $\|x\|_{\infty}=\max_{i}|x_{i}|$ is the infinity norm.
It turns out that while $\|\cdot\|_\infty$ norm estimates for the difference $Y-X\beta$ are easier to deal with, they provide sufficiently accurate information for the $\|\cdot\|_2$ norm of $Y-X\beta$ we originally care about; hence our focus on the former.
We compute the expected value of $Z_{t}$ and find a critical value $t^{*}$ such that for $t<t^*$ this expectation converges to zero.
 Combining with Markov inequality we have $\mathbb{P}\left( Z_t \geq 1 \right) \leq \mathbb{E}\left[Z_t\right] \rightarrow 0$ for all $t<t^*$ or $Z_t=0$ w.h.p. for all $t<t^*$. In particular, $t^*$ serves as a lower bound on $\min_{\beta} \|Y-X\beta\|_{\infty}$ where $\beta$ binary and $k$-sparse. This technique of finding the lower bound $t^*$ is known as the first moment method.

 We then consider the
second moment method for $Z_{t}$. In the naive form the second moment method would succeed if for $t>t^*$, $\E{Z_t^2}$ was close to $(\E{Z_t})^2$,
as in this case the Paley-Zigmund inequality would give $\pr(Z_t\ge 1)\ge \E{Z_t}^2/\E{Z_t^2} \rightarrow 1$ and therefore $t^*$ is also an upper bound for $\min_{\beta} \|Y-X\beta\|_{\infty}$. Unfortunately, the naive second moment estimation fails as it can be easily checked that for $t$ close to $t^*$, $\E{Z_t}^2/\E{Z_t^2} \rightarrow 0.$ 

We consider an appropriate conditioning to make the second moment method work. We notice that the fluctuations of $Y$ alone are enough to create a substantial gap between the two moments of $Z_t$. For this reason, we consider the conditional first and second moment of $Z_{t}$, where the conditioning is done on $Y$. The conditional
second moment involves computing large deviations estimates on a sequence of coupled bi-variate normal random variables.
A fairly detailed analysis of this large deviation estimate is obtained to arrive at the estimation of the ratio
$\E{Z_{t}|Y}^{2}/\E{Z_{t}^{2}|Y}$. We then employ the conditional version of the Paley-Zigmund inequality 
$\pr(Z_{t}\ge 1|Y)\ge \E{Z_{t}|Y}^{2}/\E{Z_{t}^{2}|Y}$ to obtain the lower bound $\pr(Z_{t}\ge 1) \geq \E {\E{Z_{t}|Y}^{2}/\E{Z_{t}^{2}|Y}} $ where expectation is taken over $Y$. Using the estimation on the lower bound we show that $t^*$, the first moment estimate, serves also as an upper bound for $\min_{\beta} \|Y-X\beta\|_{\infty}$, up to certain multiplicative constant factors.

To explain the success of the conditional technique notice that by tower property and Cauchy-Schwarz inequality
\begin{align*}
\E {\frac{\E{Z_{t}|Y}^{2}}{\E{Z_{t}^{2}|Y}}}\E{Z_{t}^{2}}  & =\E {\frac{\E{Z_{t}|Y}^{2}}{\E{Z_{t}^{2}|Y}}} \E {\E{Z_{t}^{2}|Y}} \geq\E{\E{Z_{t}|Y}  }^2=\E{Z_{t}}^2
\end{align*} 
which equivalently gives $$\E {\frac{\E{Z_{t}|Y}^{2}}{\E{Z_{t}^{2}|Y}}} \geq \frac{\E{Z_{t}}^{2}}{\E{Z_{t}^{2}}}$$certifying that the lower bound on $\pr(Z_{t}\ge 1)$ obtained through conditioning dominates the one from the direct application of Paley-Zigmund inequality.

Next we use the estimates from the Pure Noise model, for the original model involving the binary $\beta^{*}$ with $\|\beta^{*}\|_{0}=k$. We consider
the $2^k=2^{|\mathrm{Support}\left(\beta^*\right)|}$ restricted versions of the original problem of interest $(\Phi_2)$ in which the  optimization is conducted  over the space of binary $k$-sparse vectors $\beta$ where
the support of $\beta$ is constrained to intersect the support of $\beta^{*}$ in a specific way. In this form the problem can be reduced to the Pure Noise problem in a relative straightforward way (see Section \ref{section:Proof of MainResult1} for the exact reduction). This reduction alongside with the first and second moment estimates for the Pure Noise model described above allows us to approximate the optimal value of the restricted problems, and in particular of $(\Phi_2)$ as well. 

Note that conditional first and second moment methods have been used extensively in the literature (e.g. see \cite{Jiaming17}, \cite{Bollobas18} for two recent examples) but it is a common understanding that the appropriate choice of conditioning does not follow a universal reasoning. To the best of our knowledge this is the first time the conditional second moment method is used in the form described above and this might be of independent interest.

\paragraph{Organization}
The remainder of the paper is organized as follows. The description of the model, assumptions and the main results are found in the
next section. Section~\ref{section:BetaZero} is devoted to the analysis of the Pure Noise model which is also defined in this section.
Sections~\ref{section:Proof of MainResult1},~\ref{section:Problem_Phi2} and~\ref{section:OGP} are devoted to proofs of our main results.
We conclude in the last section with some open questions and directions for future research.

\section{Model and the Main Results}\label{section:MainResults}

We remind our model for convenience.
Let $X \in \mathbb{R}^{n \times p}$ be an $n \times p$ matrix with i.i.d. standard normal entries, and $W \in \mathbb{R}^p$ be a vector with i.i.d. $N\left(0,\sigma^2\right)$ entries. We also assume that $\beta^*$ is a $p \times 1$ binary vector with exactly $k$ entries equal to unity ($\beta^*$ is binary and $k$-sparse). For every binary vector $\beta\in \{0,1\}^{p}$ we let $\text{Support}(\beta):=\{i:\beta_{i}=0\}$. Namely,
$\beta_{i}=1$ if $i\in\text{Support}(\beta)$ and $\beta_{i}=0$ otherwise.
We observe $n$ noisy measurements $Y \in \mathbb{R}^n$ of the vector  $\beta^* \in \mathbb{R}^p$ given by 
\begin{align*}
Y=X\beta^*+W \in \mathbb{R}^n.
\end{align*}
Throughout the paper we are interested in the high dimensional regime where $p$ exceeds $n$ and both diverge to infinity.
Various assumptions on $k,n,p$ are required for technical reasons and some of the assumptions vary from theorem to theorem.
But almost everywhere we will be assuming that $n$ is at least of the order $k\log k$ and at most of the order $k\log p$. The results usually hold in the ``with high probability" (w.h.p.)
sense as $k,n$ and $p$ diverge to infinity, but for concreteness we usually explicitly say that $k$ diverges to infinity. This automatically 
implies the same  for $p$, since $p\ge k$, and for $n$ since it is assumed to be at least of the order  $O(k\log k)$.

In order to recover $\beta^*$, we  consider the following constrained optimization problem
\begin{align*}\begin{array}{clc} \left(\Phi_2\right) & \min  &n^{-\frac{1}{2}}||Y-X\beta||_{2} \\ &\text{s.t.}&\beta \in \{0,1\}^p  \\
&& ||\beta||_0=k.
\end{array}
\end{align*}
We denote by $\phi_{2}=\phi_{2}\left(X,W\right)$ its optimal value and by $\beta_2$ its (unique) optimal solution. Note that the solution
is indeed unique due to discreteness of $\beta$ and continuity of the distribution of $X$ and $Y$.
Namely, the optimization problem $\Phi_2$ chooses the $k$-sparse binary vector $\beta$ such that $X \beta$ 
is as close to $Y$  as possible, with respect to the $\mathbb{L}_2$ norm. 
Also note that since our noise vector, $W$, consists of i.i.d.  Gaussian entries, $\beta_2$ is also the Maximum Likelihood Estimator of $\beta^*$.

Consider now the following restricted version of the problem $\Phi_{2}$:

\begin{align*}\begin{array}{clc} \left(\Phi_2\left(\ell\right)\right) & \min  &n^{-\frac{1}{2}}||Y-X\beta||_{2} \\ &\text{s.t.}&\beta \in \{0,1\}^p  \\
&& ||\beta||_0=k, ||\beta-\beta^*||_0=2l ,
\end{array}
\end{align*}
where $\ell=0,1,2,..,k$. For every fixed $\ell$, denote by $\phi_2\left(\ell\right)$ the optimal value of $\Phi_2\left(\ell\right)$. 
$\Phi_2\left(\ell\right)$ is the problem of finding the $k$-sparse binary vector $\beta$, 
such that $X \beta$  is as close to $Y$ as possible with respect to the $\ell_2$ norm, 
but also subject to the restriction that the cardinality of the intersection of the supports of $\beta$ and $\beta^{*}$ 
is  exactly $k-\ell$. Then $\phi_{2}=\min_{\ell}\phi_{2}\left(\ell\right)$.  

Consider the extreme cases $\ell=0$ and $\ell=k$, we see that for $\ell=0$, the region that defines 
$\Phi_2\left(0\right)$ consists only of the vector $\beta^*$. On the other hand, 
for $\ell=k$, the region that defines $\Phi_2\left(k\right)$  consists of all $k$-sparse binary vectors $\beta$, whose common support with $\beta^*$
is empty.

We are now ready to state our first main result.
\begin{theorem}\label{theorem:MainResult1}
Suppose $k\log k\le Cn$ for some constant $C$ for all $k,n$. Then
\begin{enumerate}
\item [(a)] W.h.p. as $k$ increases
\begin{align}\label{eq:LowerBound}
\phi_2\left(\ell\right) \ge e^{-{3\over 2}} \sqrt{2\ell+\sigma^2}\exp \left(-\frac{\ell \log p}{n}\right),
\end{align}
for all $0\leq \ell \leq k$.

\item [(b)] Suppose further that $\sigma^2\le 2k$. Then for every sufficiently large constant $D_0$ 
if $n\le k\log p/(3\log D_0)$, then w.h.p. as $k$ increases, the cardinality of the set
\begin{align}\label{eq:UpperBound}
{\Big \{} \beta \in \{0,1\}^p: \|\beta\|_0=k, \|\beta-\beta^*\|_0=2k,~ n^{-\frac{1}{2}}\|Y-X\beta\|_2\le D_0\sqrt{2k+\sigma^2}\exp\left(-{k\log p\over n}\right){\Big\}}
\end{align}
is at least $D_0^{n\over 3}$.
In particular, this set  is exponentially large in $n$.
\end{enumerate}
\end{theorem}
The proof of this theorem is found in Section~\ref{section:Proof of MainResult1} and relies on the analysis
for the Pure Noise model developed in the next section.
The part (a) of the theorem above gives a lower bound on the optimal value of the  optimization problem $\Phi_{2}\left(\ell\right)$
for all $\ell=0,1,\ldots,k$ w.h.p. 
For this part, as stated, we only need that $k\log k\le Cn$ and $k$ diverging to infinity.
When $\ell=0$ the value of $\phi(\ell)$ is just $n^{-{1\over 2}}\sqrt{\sum_{1\le i\le n}W_i^2}$ which converges to $\sigma$ by the 
Law of Large Numbers. Note that $\sigma$ is also the value of $\sqrt{2l+\sigma^2}\exp \left(-\frac{\ell \log p}{n}\right)$ when
$\ell=0$. Thus the lower bound value in part (a) is tight up to a multiplicative constant when $\ell=0$. Importantly, as
the part (b) of the theorem
shows, the lower bound value is also tight up to a multiplicative constant when $\ell=k$, as in this case not only vectors $\beta$
achieving this bound exist, but the number of such vectors is exponentially large in $n$ w.h.p. as $k$ increases.
This result will be instrumental
for our ``all-or-nothing'' Theorem~\ref{theorem:sharptheorem} below.

Now we will discuss some implications of Theorem ~\ref{theorem:MainResult1}. The expression 
$\left(2\ell+\sigma^2\right)^{1\over 2}\exp\left(-\frac{\ell \log p}{n}\right)$,
appearing in the theorem above,
 motivates the following notation.
Let the function $\Gamma:[0,1]\rightarrow \R_+$ be defined by
\begin{align}\label{eq:Gamma_function}
\Gamma\left(\zeta\right)=
\left(2\zeta k+\sigma^2\right)^{1\over 2}\exp\left(-\frac{\zeta k \log p}{n}\right).
\end{align}
Then the lower bound (\ref{eq:LowerBound}) can be rewritten as 
\begin{align*}
\phi_2\left(\ell\right) \ge e^{-{3\over 2}} \Gamma(\ell/k).
\end{align*}
A similar inequality applies to (\ref{eq:UpperBound}).

Let us make some immediate observations regarding the function $\Gamma$. It
is a strictly log-concave function in $\zeta \in [0,1]$:
\begin{align*}
\log \Gamma\left(\zeta\right)=\frac{1}{2} \log \left(2\zeta k+\sigma^2\right)-\zeta \frac{k \log p}{n}.
\end{align*}
and hence 
\begin{align*}
\min_{0\le \zeta\le 1}\Gamma\left(\zeta\right)=\min\left(\Gamma\left(0\right),\Gamma\left(1\right)\right)=\min\left(\sigma,\sqrt{2k+\sigma^2} \exp\left(-\frac{k \log p}{n}\right)\right).
\end{align*}
Now combining this observation with the results of Theorem~\ref{theorem:MainResult1} 
we obtain as a corollary a tight up to a multiplicative constant approximation
of the value $\phi_2$ of the optimization problem $\Phi_2$.

\begin{theorem}\label{theorem:Value-phi2}
Under the assumptions of parts (a) and (b) of Theorem\ref{theorem:MainResult1}, for every $\epsilon>0$ and for every sufficiently large constant $D_0$ 
if $n\le k\log p/(3\log D_0)$, then w.h.p. as $k$ increases, 
\begin{align*}
e^{-{3\over 2}}\min\left(\sigma,\sqrt{2k+\sigma^2}\exp\left(-{k\log p\over n}\right)\right)
\le
\phi_2\
\le \min\left((1+\epsilon)\sigma,D_0\sqrt{2k+\sigma^2}\exp\left(-{k\log p\over n}\right)\right).
\end{align*}
\end{theorem}

\begin{proof}
By Theorem~\ref{theorem:MainResult1} we have that $\phi_2$ is at least 
\begin{align*}
e^{-{3\over 2}}\min_\zeta\Gamma(\zeta)=e^{-{3\over 2}}\min\left(\Gamma(0),\Gamma(1)\right).
\end{align*}
This establishes the lower bound. For the upper bound we have $\phi_2\le \min(\phi_2(0),\phi_2(k))$. By the Law of Large Numbers,
$\phi_2(0)$ is at most $(1+\epsilon)\sigma$ w.h.p. as $k$ (and therefore $n$) increases. The second part of Theorem\ref{theorem:MainResult1}
gives provides the necessary bound on $\phi_2(k)$.
\end{proof}

As in the introduction, letting $n^*=\frac{2k \log p}{\log \left( \frac{2k}{\sigma^2}+1\right)}$, 
we conclude that $\min_{\zeta}\Gamma\left(\zeta\right)=\Gamma\left(1\right)$ when 
$n<n^{*}$ and $=\Gamma\left(0\right)$ when $n>n^{*}$, with the critical case $n=n^{*}$ (ignoring the integrality of $n^{*}$), giving
$\Gamma\left(0\right)=\Gamma\left(1\right)$. This observation suggests the following \textbf{``all-or-nothing''} type behavior of the problem $\Phi_{2}$,
if $\Gamma$ was an accurate estimate of the value of the optimization problem $\Phi_2$.
When $n>n^{*}$ the solution $\beta_{2}$ of the minimization problem $\Phi_{2}$ is expected to coincide with the ground truth
$\beta^{*}$ since in this case $\zeta=0$, which corresponds to $\ell=0$, minimizes $\Gamma\left(\zeta\right)$. On the other hand, when
$n<n^{*}$, the solution $\beta_{2}$ of the minimization problem $\Phi_{2}$ is not even expected to have any common support with the 
ground truth $\beta^{*}$, as in this case $\zeta=1$, which corresponds to $\ell=k$, minimizes $\Gamma\left(\zeta\right)$. Of course, this is nothing
more than just a suggestion, since by Theorem~\ref{theorem:MainResult1},  $\Gamma\left(\zeta\right)$ 
only provides a lower and upper bounds on the optimization problem $\Phi_{2}$, which  tight only up to a multiplicative constant.
Nevertheless, we can turn this observation into a theorem, which is our second main result.

\begin{theorem}\label{theorem:sharptheorem} 
Let $\epsilon>0$ be arbitrary. 
Suppose $\max \{k,\frac{2k}{\sigma^2}+1\} \leq \exp \left(  \sqrt{ C \log p } \right)$, for some $C>0$ for all $k$ and $n$. Suppose furthermore that $k\rightarrow \infty$ and $\sigma^2/k\rightarrow 0$
as $k\rightarrow\infty$. 
If $n\ge \left(1+\epsilon\right) n^*$, then w.h.p. as $k$ increases
\begin{align*}
\frac{1}{2k}\|\beta_{2}-\beta^{*}\|_{0} \rightarrow 0.
\end{align*}
On the other hand if
$\frac{1}{C} k \log k \le n\le \left(1-\epsilon\right)n^*$,
then w.h.p. as $k$ increases
\begin{align*}
\frac{1}{2k}\|\beta_{2}-\beta^{*}\|_{0}  \rightarrow 1.
\end{align*}
\end{theorem}

The proof of  Theorem~\ref{theorem:sharptheorem} is found in Section~\ref{section:Problem_Phi2}. 
The theorem above confirms the ``all-or-nothing'' type behavior of the optimization problem $\Phi_{2}$, depending on 
how $n$ compares with $n^{*}$. Recall that, according to~\cite{wang2010information}, $n^*$ is an information 
theoretic lower bound for recovering $\beta^*$ from $X$ and $Y$ \emph{precisely}, and also for $n<n^*$ it does 
not rule out the possibility of recovering at least a fraction of bits of $\beta^*$. Our theorem however shows firstly that $n^*$ is exactly the infortmation theoretic threshold for exact recovery and also that if $n<n^*$
the optimization problem $\Phi_2$ fails to recover asymptotically any of the bits of $\beta^*$. We note also that
the value of $n^*$ is naturally larger than the corresponding threshold when $k=1$, namely $2\log p/\log(1+2\sigma^{-2})$,
which is asymptotically $\sigma^2\log p=n_{{\rm inf,1}}$. Interestingly, however this value for $n$, which has appeared also, as explained in the Introduction, as a weaker information theoretic bound also marks a phase transition point as we discuss in the proposition below.

As our result  above shows, the recovery of $\beta^{*}$ is possible by solving $\Phi_{2}$ (say by running the integer programming problem)
when $n>n^{*}$, even though efficient algorithms such as compressive sensing and LASSO algorithms are only known to work when
$n\ge (2k+\sigma^2)\log p$.
This suggests that the region $n\in [n^{*}, (2k+\sigma^2)\log p]$ might correspond to solvable but algorithmically hard regime
for the problem of finding $\beta^{*}$.

We turn to the study of the ``limiting curve" $\Gamma\left(\zeta\right)$. Note that we refer to the curve $\Gamma=\Gamma(\zeta)$ as the limiting curve because (1) from Theorem~\ref{theorem:MainResult1} for $\zeta:=\ell/k$ it provides a deterministic lower bound for the value of $\phi_2(\ell)$ (up to a multiplicative constant) but also an upper bound when $\ell=0$ and $\ell=k$  (up to a multiplicative constant) and (2) from the example of Theorem~\ref{theorem:sharptheorem}  it appears that structural properties of the curve $\Gamma$, such as the behavior of its maximum argument, accurately suggest a similar behavior for the actual values of $\phi_2(\ell)$,

Studying the properties of the "limiting curve" $\Gamma\left(\zeta\right)$ we discover an intriguing link between its behavior and the three
fundamental thresholds discussed above. Namely, the threshold $n_{{\rm inf,1}}=\sigma^{2}\log p$, the threshold 
$n^{*}=\frac{2k}{\log \left( \frac{2k}{\sigma^2}+1\right)}\log p$, and finally the threshold $n_{\text{LASSO/CS}}=(2k+\sigma^2)\log p$. 
For the illustration of different cases outlined in the proposition above see Figure~\ref{fig:subfigures}.

\begin{proposition}\label{prop:GammaMonotonic}
The function $\Gamma$ satisfies the following properties.
\begin{enumerate}
\item When $n\le \sigma^2 \log p$, $\Gamma$  is a strictly decreasing function of $\zeta$. (Figure 1(a)),
\item When $\sigma^2 \log p< n < n^*$, $\Gamma$ is not monotonic and it attains its minimum at $\zeta=1$. (Figure 1(b)),
\item When $n=n^*$, $\Gamma$ is not monotonic and it attains its minimum at $\zeta=0$ and $\zeta=1$. (Figure 1(c))
\item When $n^*<n< (2k+\sigma^{2}) \log p$, $\Gamma$ is not monotonic and it attains its minimum at $\zeta=0$. (Figure 1(d))
\item When $n>\left(2k+\sigma^2\right) \log p$, $\Gamma$ is a strictly increasing function of $\zeta$. (Figure 1(d))
\end{enumerate}
\end{proposition}
In particular, we see that both the  bound $n_{\rm inf,1}=\sigma^{2}\log p$, and  $n_{\rm LASSO/CS}=\left(2k+\sigma^2\right) \log p$ 
mark the phase transition change of (lack of) monotonicity property of the limiting curve
$\Gamma$. We also summarize our findings in Table~\ref{table:Table1}. The proof of this proposition is found in Section~\ref{section:Problem_Phi2}.

To study the apparent algorithmic hardness of the problem in the regime $n\in [n_{\rm inf,1},n_{\rm LASSO/CS}]$,
as well as to see whether the picture suggested by the curve $\Gamma$ is actually accurate,
we now study the geometry of the solution space of the problem $\Phi_{2}$. We establish in particular, that the solutions $\beta$ which
are sufficiently ``close'' to optimality in $\Phi_2$, that is the $\beta$'s which have objective value $\|Y-X\beta\|_{2}$ close to the optimal value $\phi_2$, break into two separate clusters; namely those which have a ``large'' overlap with $\beta^*$, 
and those which are far from it, namely those which have a ``small'' overlap with $\beta^*$. 
As discussed in  Introduction, such an Overlap Gap Property (OGP) appears to mark the onset of algorithmic hardness for many
randomly generated constraint satisfaction problems. Here we demonstrate its presence in the context of high dimensional regression problems.

\begin{figure}[ht!]
     \begin{center}
        \subfigure[The behavior of $\Gamma$ for $n=10<\sigma^2\log p$.]{%
           \label{fig:first}
            \includegraphics[width=0.4\textwidth]{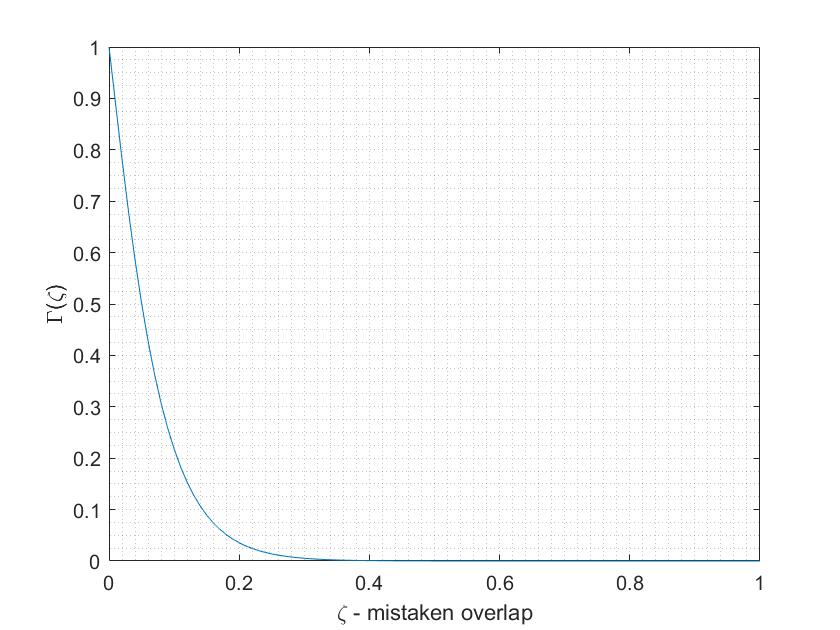}
        }%
        \subfigure[The behavior of $\Gamma$ for $\sigma^2 \log p<n=120<n^*$.]{%
           \label{fig:second}
           \includegraphics[width=0.4\textwidth]{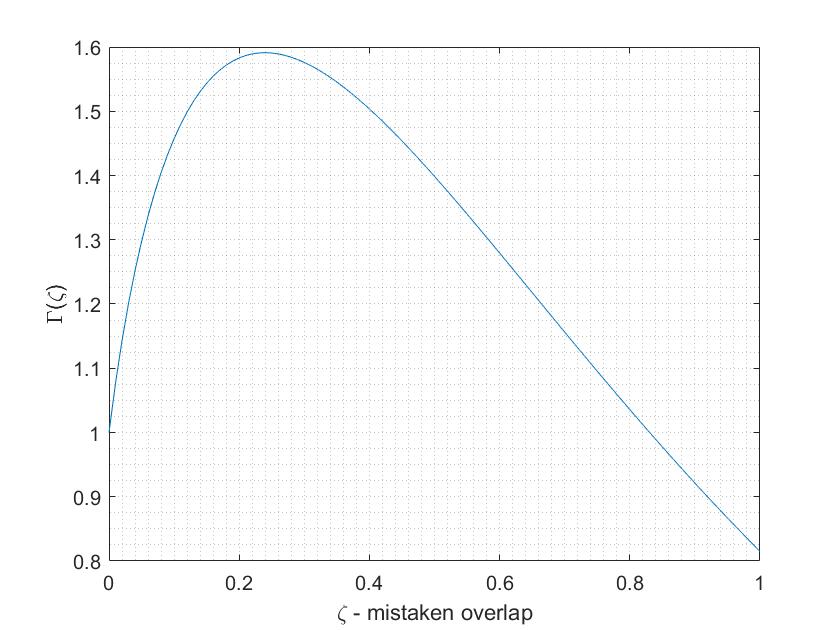}
        }\\ 
        \subfigure[The behavior of $\Gamma$ for $n=136=n^*$.]{%
            \label{fig:third}
            \includegraphics[width=0.4\textwidth]{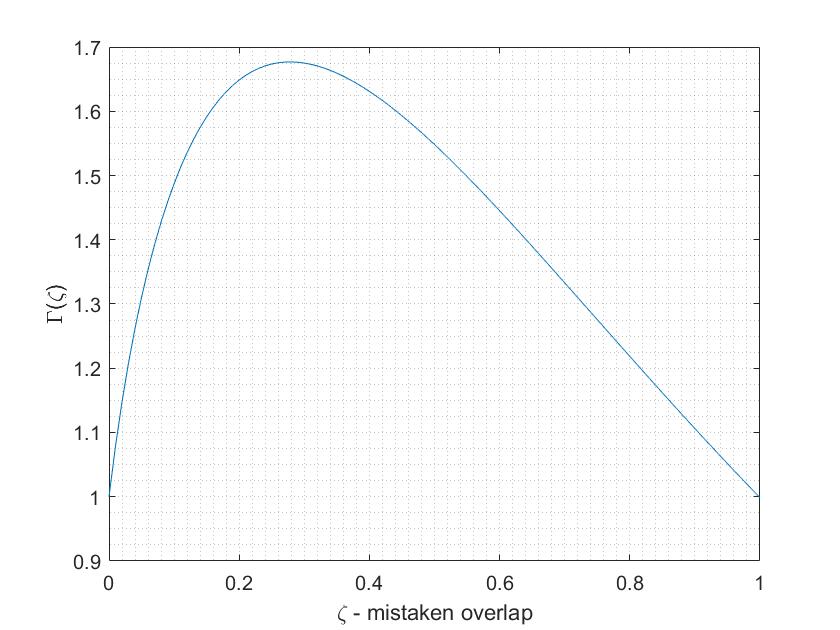}
        }%
        \subfigure[The behavior of $\Gamma$ for $n^*<n=200<(2k+\sigma^2) \log p$.]{%
            \label{fig:fourth}
            \includegraphics[width=0.4\textwidth]{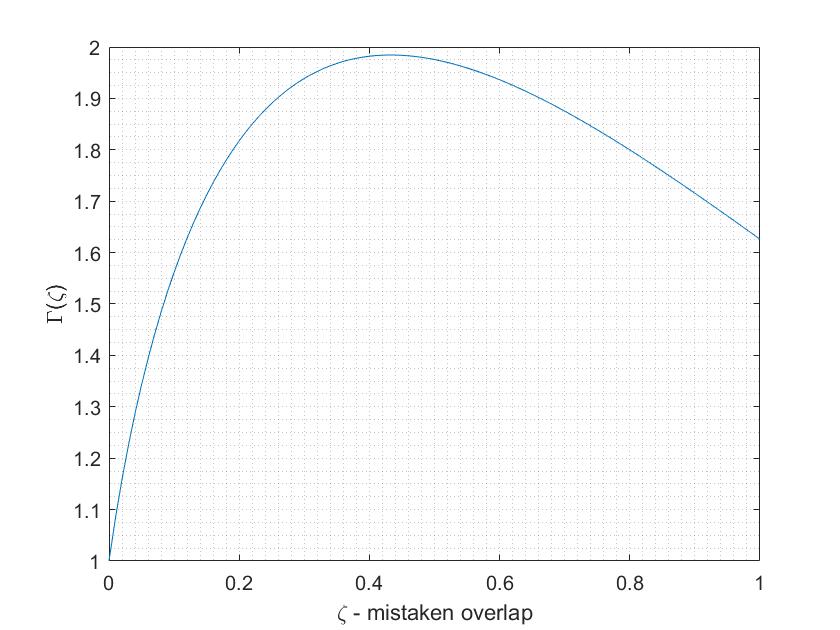}
        }\\%
        \subfigure[The behavior of $\Gamma$ for $(2k+\sigma^2) \log p<n=450$.]{%
            \label{fig:fifth}
            \includegraphics[width=0.4\textwidth]{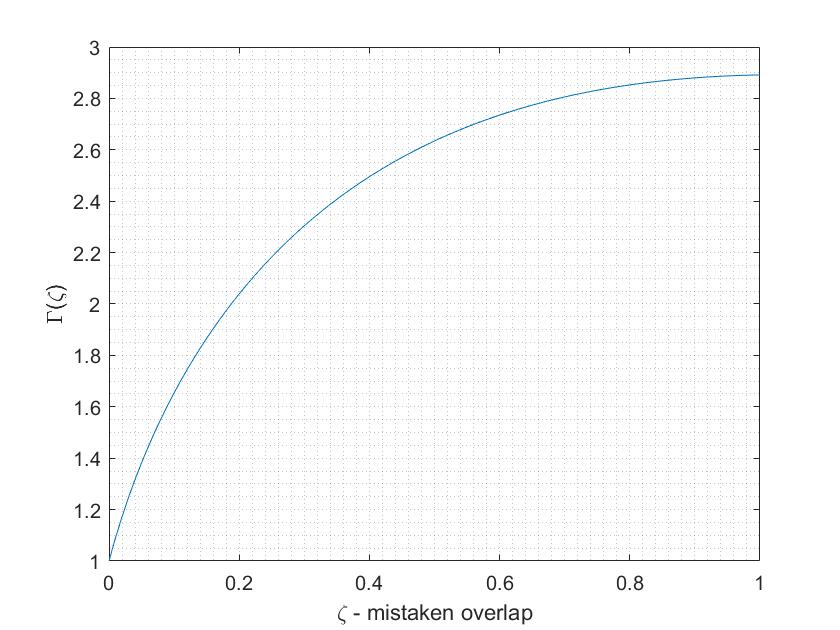}
        }%
    \end{center}
    \caption{%
       The five different phases of the function $\Gamma$ as $n$ grows. We consider the case when $p=10^9, k=10$ and $\sigma^2=1$. In this case $\left \lceil{\sigma^2 \log p}\right \rceil=21,\left \lceil{n^*}\right \rceil=137$ and $\left \lceil{(2k+\sigma^2) \log p}\right \rceil=435$.
     }%
   \label{fig:subfigures}
\end{figure}


\begin{table}
\begin{tabular}{| c | l |}
\hline
$n<n_{\text{inf},1}$       
& $\Gamma$ is monotonically decreasing                                           \\ 
\hline
$n_{\text{inf},1}< n<n^*$ &  $\Gamma$ is not monotonic                                \\ 
 & and attains its minimum at $\zeta=1$ \\ \hline
$n^*<n<n_{\text{LASSO/CS}}$        & $\Gamma$ is not monotonic   \\
&   and  attains its minimum at $\zeta=0$                           \\ 
                              \hline
$n_{\text{LASSO/CS}}<n$            & $\Gamma$ is monotonically increasing                           \\ 

\hline
\end{tabular}
\centering
\caption{The phase transition property of the limiting curve $\Gamma\left(\zeta\right)$} \label{table:Table1}
\end{table}

The presence of the OGP is indeed suggested by the lack of monotonicity of the limiting curve $\Gamma$ when $\sigma^{2}\log p<n<(2k+\sigma^2)\log p$. 
Indeed, in this case fixing any value  $\gamma$ strictly smaller than the maximum value of $\Gamma$, but larger than 
both $\Gamma\left(0\right)$ and $\Gamma\left(1\right)$, we see that set of overlaps $\zeta$ achieving value $\le \gamma$ is disjoint union of two intervals
of the form $[0,\zeta_{1}]$ and $[\zeta_{2},1]$ with $\zeta_{1}<\zeta_{2}$.
Of course, as before this is nothing but a suggestion, since the function
$\Gamma$ is only a lower bound on the objective value $\Phi_{2}(\ell)$ for $\zeta=\ell/k$. In the next theorem we establish that the OGP
indeed takes place, in the case where $n$ is between the information-theoretic threshold $n^*$ and a constant multiple of $n_{\text{LASSO/CS}}$. The case where $n$ lies between $\sigma^2 \log p$ and $n^*$ is discussed subsequent to the statement of the Theorem. 
Given any $r\ge 0$, let
\begin{align*}
S_{r} := \{ \beta \in \{0,1 \}^p : ||\beta||_0=k, n^{-\frac{1}{2}}||Y-X \beta||_2 < r \}.
\end{align*}

\begin{theorem}[The Overlap Gap Property]\label{theorem:OGP}{theorem:sharptheorem} 
Suppose the assumptions of Theorem~\ref{theorem:MainResult1} hold and for some $C>0$, $k\log k \leq Cn$.
For every sufficiently large constant $D_0$ there exist sequences $0<\zeta_{1,k,n}<\zeta_{2,k,n}<1$ satisfying
\begin{align*} 
\lim_{k\rightarrow\infty}k\left(\zeta_{2,k,n}-\zeta_{1,k,n}\right)= + \infty,
\end{align*}
as $k\rightarrow\infty$, and such that if $r_k=D_0\max\left(\Gamma(0),\Gamma(1)\right)$ and
$n^* \le n\le k\log p/(3\log D_0)$ then w.h.p. as $k$ increases the following holds

\begin{itemize}
\item[(a)] For every $\beta \in S_{r_k}$ 
\begin{align*}
\left(2k\right)^{-1}\|\beta-\beta^{*}\|_{0}<\zeta_{1,k,n} \text{ or }
\left(2k\right)^{-1}\|\beta-\beta^{*}\|_{0}>\zeta_{2,k,n}.
\end{align*}

\item[(b)]  $\beta^* \in S_{r_k}$. In particular the set 
\begin{align*}
S_{r_{k}}\cap \{\beta: \left(2k\right)^{-1}\|\beta-\beta^{*}\|_0<\zeta_{1,k,n} \}
\end{align*}
is non-empty.

\item[(c)] The cardinality of the set
\begin{align*}
|S_{r_{k}}\cap \{\beta: \|\beta-\beta^{*}\|_0\}=2k\}|,
\end{align*}
is at least $D_0^{n\over 3}$. In particular  the set $S_{r_{k}}\cap \{\beta: \|\beta-\beta^{*}\|_0\}=2k\}$ has exponentially
many in $n$ elements.
\end{itemize}
\end{theorem}

The proof of Theorem~\ref{theorem:OGP} is found in Section~\ref{section:OGP}.
The property $k\left(\zeta_{2,k,n}-\zeta_{1,k,n}\right) \rightarrow \infty$ in the statement of the theorem implies in particular that 
the difference $\left(\zeta_{2,k,n}-\zeta_{1,k,n}\right)$ grows faster than $1/k$ as
$k$ diverges, ensuring that for many overlap values $\ell$, the ratio $2\ell/k$ falls within
the interval $[\zeta_{1,k,n},\zeta_{2,k,n}]$. Namely, the overlap gap interval is non-vacuous for all large enough $k$. Note that for $k$ such that $\max \{k ,\frac{2k}{\sigma^2}+1 \} \leq \exp \left( \sqrt{C \log p} \right)$ for large $k$ it holds $\frac{1}{C}k \log k <n^*$ and in particular the result of Theorem~\ref{theorem:OGP} holds for all $n \in [n^*,k\log p/(3\log D_0)]$ w.h.p. since the constraint $k \log k \leq Cn$ becomes redundant.

The study of Overlap Gap Property in the case where $\sigma^2 \log p <n<n^*$ does not have a clear algorithmic value, since the problem becomes information-theoretic impossible. Nevertheleess, the first moment curve is also non-monotonic in that regime suggesting that the Overlap Gap Property still holds. Under the additional stringent assumption that $\sigma^2 \rightarrow + \infty$ as $k \rightarrow + \infty$  it can be established that Overlap Gap Property indeed appears in that regime. The proof follows by almost identical arguments with the proof of Theorem \ref{theorem:OGP} by setting $\zeta_{1,k,n}:=\frac{e^7D_0^2 \sigma^2}{2k}$ and $\zeta_{2,k,n}:=\frac{e^7D_0^2 \sigma^2}{k}$. 

%
%
%
%
%

\section{The Pure Noise Model}\label{section:BetaZero}
In this subsection we consider a modified model corresponding to the case $\beta^{*}=0$, which we dub as pure noise model. This model serves
as a technical building block towards proving Theorem ~\ref{theorem:MainResult1}. The model is described as follows.

\subsection*{The Pure Noise Model}

Let $X \in \mathbb{R}^{n \times p}$ be an $n \times p$ matrix with i.i.d. standard normal entries, and $Y \in \mathbb{R}^n$ be a vector with i.i.d. $N\left(0,\sigma^2\right)$ entries.  $Y,X$ are  independent. We study the optimal value $\psi_2$ of the following optimization problem: 

$$\begin{array}{clc} \left(\Psi_2\right) & \min  &n^{-\frac{1}{2}}  ||Y-X\beta||_2 \\ &\text{s.t.}&\beta \in \{0,1\}^p  \\
&& ||\beta||_0=k. 
\end{array}$$
That is, we no longer have ground truth vector $\beta^{*}$, and instead search for a vector $\beta$ which makes $X\beta$ as close to an independent
vector $Y$ as possible in $\|\cdot\|_{2}$ norm.

We now state our main result for the pure noise model case. 
\begin{theorem}\label{theorem:PureNoiseLowerBound}
The following holds for all $n,p,k,\sigma$:
\begin{align}
\mathbb{P}\left(\psi_2  \ge e^{-3/2}{\sqrt{k+\sigma^2}\exp \left(-\frac{k \log p}{n}\right)}\right) \geq 1-e^{-n}. \label{eq:FirstMomentBound}
\end{align}
Furthermore, for every $C>0$ and every sufficiently large constant $D_0$,  if $k\log k\le Cn$,
$k\le \sigma^2\le 3k$, and $n\le k\log p/(2\log D_0)$, the cardinality of the set
\begin{align*}
{\Big\{}\beta \in \{0,1\}^p: \|\beta\|_0=k, n^{-{1\over 2}}\|Y-X\beta\|_2\le D_0\sqrt{k+\sigma^2}\exp\left(-{k\log p\over n}\right){\Big\}}
\end{align*}
is at least $D_0^{n\over 3}$ w.h.p. as $k\rightarrow\infty$.
\end{theorem} 
In the theorem above the value  of the constant $D_0$ may depend on $C$ (but does not depend on any other parameters, such as $n,p$ or $k$).
We note that in the second part of the theorem, our assumption $k\rightarrow\infty$ by our other assumptions also implies that both $n$ and $p$
diverge to infinity. 
The theorem above says that the value $\sqrt{k+\sigma^2}\exp\left(-{k\log p\over n}\right)$ is the tight value of $\psi_2$ 
for the optimization  problem $\Psi_2$, up to a multiplicative constant. Moreover, for the upper bound part, according to the second part 
of the theorem,  the number
of solutions achieving asymptotically this value is exponentially large in $n$. The assumption $k\le \sigma^2\le 3k$
is adopted so that the result of the theorem is transferable to the original model where $\beta^*$ is a $k$-sparse binary vector,
in the way made precise in the following section.

The proof of Theorem~\ref{theorem:PureNoiseLowerBound} is the subject of this section. The lower bound is obtained by a simple moment
argument. The upper bound is the part which consumes the bulk of the proof and will employ a certain conditional second moment method.
Since for any $x\in \R^n$ we have $n^{-{1\over 2}}\|x\|_2\le \|x\|_\infty$, the result will be implied by looking instead at the
cardinality of the set
\begin{align}\label{eq:SwitchNorm}
{\Big\{}\beta \in \{0,1\}^p:\|\beta\|_0=k, \|Y-X\beta\|_\infty\le D_0\sqrt{k+\sigma^2}\exp\left(-{k\log p\over n}\right){\Big\}},
\end{align}
and establishing the same result for this set.

\subsection{The Lower Bound. Proof of (\ref{eq:FirstMomentBound}) of Theorem~\ref{theorem:PureNoiseLowerBound}}

\begin{proof}
Observe that  $ p^k \geq \binom{p}{k}$ implies  $\exp\left( \frac{k  \log p}{n}\right) \geq \binom{p}{k}^\frac{1}{n}$ and therefore
\begin{align*}
 \mathbb{P}\left( \psi_2\ge  e^{-\frac{3}{2}}\exp \left(-\frac{k \log p}{n}\right)\sqrt{k+\sigma^2}\right)  
 \geq \mathbb{P}\left( \psi_2\ge e^{-\frac{3}{2}}\binom{p}{k}^{-{1\over n}}\sqrt{k+\sigma^2}\right).
\end{align*} 
Thus it suffices to show
\begin{align*}
\mathbb{P}\left( \psi_2\ge e^{-\frac{3}{2}}\binom{p}{k}^{-{1\over n}}\sqrt{k+\sigma^2}\right)\ge 1-e^{-n}.
\end{align*}
Given any $t>0$, let 
\begin{align*}
Z_t&=| \{ \beta \in \{0,1\}^p  : |\beta||_0=k,n^{-\frac{1}{2}}||Y-X\beta||_{2}<t \}|\\
&=\sum_{\beta \in \{0,1\}^p |, |\beta||_0=k } {\bf 1}\left(n^{-\frac{1}{2}}||Y-X\beta||_{2}<t\right),
\end{align*}  ${\bf 1}\left(A\right)$ denotes the indicator function applied to the event $A$. 
Let $t_0:=e^{-\frac{3}{2}} \binom{p}{k}^{-\frac{1}{n}}$. Observe that $t_0 \in \left(0,1\right)$. We  have

\begin{align*}
\mathbb{P}\left(\psi_2<e^{-\frac{3}{2}} \binom{p}{k}^{-\frac{1}{n}} \sqrt{k+\sigma^2}\right) 
&= \mathbb{P}\left(Z_{t_0\sqrt{k+\sigma^2}} \geq 1\right) \\
&\leq \E{Z_{t_0\sqrt{k+\sigma^2}}}.
\end{align*}

Now notice that $Z_{t_0\sqrt{k+\sigma^2}}$ is a sum of the $\binom{p}{k}$ indicator variables, each one of them referring to the event that a specific $k$-sparse binary $\beta$ satisfies $n^{-\frac{1}{2}}||Y-X\beta||_2<t_0\sqrt{k+\sigma^2}$ namely it satisfies  $||Y-X\beta||_2^2<t_0^2\left(k+\sigma^2\right)n $.

Furthermore, notice that for fixed $\beta \in \{0,1 \}^p$ and $k$-sparse, $Y-X\beta=Y-\sum_{i \in S} X_i$ for $S\triangleq\mathrm{Support}\left(\beta\right)$,
where $X_{i}$ is the $i$-th column of $X$.
Hence since $Y,X$ are independent, $Y_i$ are i.i.d. $N\left(0,\sigma^2\right)$ and $X_{i,j}$ are i.i.d. 
$N\left(0,1\right)$, then $||Y-X\beta||_2^2$ is distributed as
$\left(k+\sigma^2\right)\sum_{i=1}^{n}Z_i^2$ where $Z_i$ i.i.d. standard normal Gaussian, namely $\left(k+\sigma^{2}\right)$ multiplied
by a random variable with chi-squared distribution with $n$ degrees of freedom. 
Hence for a fixed k-sparse $\beta \in \{0,1\}^p$, after rescaling, it holds 
\begin{align*}
\mathbb{P}\left(||Y-X\beta||_2n^{-\frac{1}{2}}<t_0\sqrt{k+\sigma^2}\right)=\mathbb{P}\left(\sum_{i=1}^{n}Z_i^2 \leq t_0^2  n\right).
\end{align*} 
Therefore  
\begin{align*}
&\E{Z_{t_0\sqrt{k+\sigma^2}}}=\E{\sum_{\beta \in \{0,1\}^p |, |\beta||_0=k } 1\left(n^{-\frac{1}{2}}||Y-X\beta||_{2}<t\right)}\\
&=\binom{p}{k}\mathbb{P}\left(||Y-X\beta||_2n^{-\frac{1}{2}}<t_0\sqrt{k+\sigma^2}\right)\\
&=\binom{p}{k}\mathbb{P}\left(\sum_{i=1}^{n}Z_i^2 \leq t_0^2  n\right).
\end{align*} 
We conclude
\begin{equation}
\mathbb{P}\left(\psi_2<e^{-\frac{3}{2}} \binom{p}{k}^{-\frac{1}{n}} \sqrt{k+\sigma^2}\right) \leq \E{Z_{t_0\sqrt{k+\sigma^2}}}=\binom{p}{ k} \mathbb{P}\left(\sum_{i=1}^{n}Z_i^2 \leq t_0^2  n\right).
\label{eq:1keypr}
\end{equation}  
Using standards  large deviation theory estimates (see for example~\cite{large_deviations_ShWeiss}), 
for the sum of $n$ chi-square distributed random variables we obtain that for $t_0 \in \left(0,1\right)$,
\begin{equation}
\mathbb{P}\left(\sum_{i=1}^{n}Z_i^2 \leq n  t_0^2\right)\leq \exp \left(nf\left(t_0\right) \right)
\label{eq:larged}
\end{equation}
 with $f\left(t_0\right)\triangleq\frac{1-t_0^2+2\log \left(t_0\right)}{2}$.

Since $f\left(t_0\right)<\frac{1}{2}+\log t_0$, and as we recall 
$t_0=e^{-\frac{3}{2}} \binom{p}{k}^{-\frac{1}{n}}<1$ we obtain,
\begin{align*}
&f\left(t_0\right)<-1-\frac{1}{n} \log \binom{p}{k},
\end{align*}
which implies
\begin{align*}
\exp \left(n f\left(t_0\right)\right)<\exp\left(-n\right) \binom{p}{k}^{-1},
\end{align*}
which implies
\begin{align*}
 \binom{p}{k}\exp \left(n f\left(t_0\right)\right)<\exp\left(-n\right).
\end{align*}
Hence using the above inequality, (\ref{eq:larged}) and (\ref{eq:1keypr}) we get 
$$ 
\mathbb{P}\left(\psi_2<e^{-\frac{3}{2}} \binom{p}{k}^{-\frac{1}{n}} \sqrt{k+\sigma^2}\right)   \leq \exp\left(-n \right),
$$ and the proof of (\ref{eq:FirstMomentBound}) is complete.

\end{proof}

We now turn to proving the upper bound part  of Theorem~\ref{theorem:PureNoiseLowerBound}. We begin by establishing several 
preliminary results.

\subsection{Preliminaries}
We first observe that $k\log k\le Cn$ and $n\le k\log p/(2\log D_0)$, implies $\log k\le C\log p/(2\log D_0)$. In particular,
for $D_0$ sufficiently large
\begin{align}\label{eq:k-less-p}
k^4\le p.
\end{align}

We establish the following two auxiliary lemmas.

\begin{lemma}
\label{binomiallemma}
If $m_1,m_2 \in \mathbb{N}$ with $m_1 \geq 4$ and $m_2 \leq \sqrt{m_1}$ then $$\binom{m_1}{m_2} \geq \frac{m_1^{m_2}}{4 m_2!}.$$
\end{lemma}

\begin{proof}
We have,
\begin{align*}
\binom{m_1}{m_2} \geq \frac{m_1^{m_2}}{4 m_2!}
\end{align*}  
holds if an only if
\begin{align*}
\prod_{i=1}^{m_2-1}\left(1-\frac{i}{m_1}\right) \geq \frac{1}{4}.
\end{align*}
Now  $m_2 \leq \sqrt{m_1}$ 
implies 
\begin{align*}
\prod_{i=1}^{m_2-1}\left(1-\frac{i}{m_1}\right) &\geq \prod_{i=1}^{\left \lfloor{\sqrt{m_1}} \right \rfloor}\left(1-\frac{i}{m_1}\right) \\
&\ge \left(1-\frac{1}{\sqrt{m_1}}\right)^{\sqrt{m_1}},\end{align*} 
It is easy to verify that  $x \geq 2$ implies  $\left(1-\frac{1}{x}\right)^x \geq \frac{1}{4}$. This  completes the proof.

\end{proof}

\begin{lemma}
\label{concavelemma}
The function $f: [0,1) \rightarrow \mathbb{R}$ defined by $$f\left(\rho\right):=\frac{1}{\rho} \log \left( \frac{1-\rho}{1+\rho}\right),$$ for $\rho \in [0,1)$ is concave.

\end{lemma}

\begin{proof}
The second derivative of $f$ equals $$\frac{2\left(-4\rho^3+\left(\rho^2-1\right)^2\log \left(\frac{1-\rho}{1+\rho} \right)+2\rho \right)}{\rho^3\left(1-\rho^2\right)^2}.$$ Hence, it suffices to prove that the function $g: [0,1) \rightarrow \mathbb{R}$ defined by $$g\left(\rho\right):=-4\rho^3+\left(\rho^2-1\right)^2\log \left(\frac{1-\rho}{1+\rho} \right)+2\rho$$ is non-positive. But for $\rho \in [0,1)$ $$g'\left(\rho\right)=4 \rho \left(1-\rho^2\right) \log \left(\frac{1+\rho}{1-\rho}\right)-10\rho^2 \text{ and }g''\left(\rho\right)=4 \left( \left(1-3\rho^2\right) \log \left(\frac{1+\rho}{1-\rho}\right)-3\rho \right).$$ We claim the second derivate of $g$ is always negative.
If $1-3\rho^2<0$, then $g''\left(\rho\right)<0$ is clearly negative. Now suppose
$1-3 \rho^2>0$. The inequality $\log \left(1+x\right) \leq x$ implies $\log \left(\frac{1+\rho}{1-\rho}\right) \leq \frac{2\rho }{1-\rho}$. 
Hence, 
\begin{align*}
g''\left(\rho\right) \leq 4 \left(\frac{2\rho }{1-\rho} \left(1-3 \rho^2\right)-3\rho \right) =4 \rho \frac{3\rho-6 \rho^2-1}{1-\rho}<0,
\end{align*}
where the last inequality follows from the fact that
$3\rho-6 \rho^2-1<0$ for all $\rho \in \mathbb{R}$.
 
Therefore $g$ is concave and therefore $g'\left(\rho\right) \leq g'\left(0\right)=0$ which implies that $g$ is also decreasing. In particular for all $\rho \in [0,1)$, $g\left(\rho\right) \leq g\left(0\right)=0$. 
\end{proof}

For any $t>0, y \in \mathbb{R}$ and  a standard Gaussian random variable $Z$ we let
\begin{equation}
p_{t,y}:= \mathbb{P}\left(|Z-y| \leq t\right).
\label{eq:pdef}
\end{equation} 
Observe that
\begin{align*}
p_{t,y}=\int_{[-t,t]}{1\over \sqrt{2\pi}}e^{-{(y+x)^2\over 2}}dx\ge \sqrt{2\over \pi}te^{-{y^2+t^2\over 2}},
\end{align*}
leading to 
\begin{align}\label{eq:bound-log-p}
\log p_{t,y}\ge \log t-{t^2\over 2}-{y^2\over 2}+(1/2)\log(2/\pi).
\end{align}

Similarly, for any $t>0, y \in \mathbb{R}, \rho \in [0,1]$ we let 
\begin{align}
q_{t,y,\rho}:= \mathbb{P}\left(|Z_1-y| \leq t,|Z_2-y| \leq t\right), \label{eq:qdef}
\end{align} 
where the random pair $\left(Z_1,Z_2\right)$ follows a bivariate normal distribution with correlation $\rho$. 
In particular, $q_{t,y,0}=p_{t,y}^2$ and  $q_{t,y,1}=p_{t,y}$.
We now state and prove  a lemma which provides an upper bound on the ratio $\frac{q_{t,y,\rho}}{p^2_{t,y}},$ for any $\rho \in [0,1)$.

\begin{lemma}\label{lemma:q-in-p}
For any $t>0, y \in \mathbb{R},\rho \in [0,1)$,
\begin{align*}\frac{q_{t,y,\rho }}{p_{t,y}^2} \leq  \sqrt{\frac{1+\rho}{1-\rho}}  e^{\rho y^2}.
\end{align*}
\end{lemma}

\begin{proof}
We have 
\begin{align*}
q_{t,y,\rho }&=\frac{1}{2 \pi \sqrt{1-\rho^2}}\int_{[y-t,y+t]^2} \exp \left(-\frac{x^2+z^2-2\rho x z}{2\left(1-\rho^2\right)}\right)dxdz \\
&=\frac{1}{2 \pi \sqrt{1-\rho^2}}\int_{[y-t,y+t]^2} \exp \left(-\frac{\left(x- \rho z\right)^2}{2\left(1-\rho^2\right)}-\frac{z^2}{2}\right)dxdz \\
&\leq \frac{1}{2 \pi \sqrt{1-\rho^2}} \int_{[y-t,y+t]} \exp \left(-\frac{x_2^2}{2} \right)dx_2 
\int_{[y\left(1-\rho\right)-t\left(1+\rho\right),y\left(1-\rho\right)+t\left(1+\rho\right)]} 
\exp \left(-\frac{x_1^2}{2\left(1-\rho^2\right)} \right)dx_1,
\end{align*} 
where in the inequality we have introduced 
the change of variables $\left(x_1,x_2\right) =\left(x-\rho z, z\right)$ 
and upper bounded the transformed domain by 
\begin{align*}
[y\left(1-\rho\right)-t\left(1+\rho\right),y\left(1-\rho\right)+t\left(1+\rho\right)] \times [y-t,y+t].
\end{align*}
Introducing another change of variable $x_1 = x_3\left(1+\rho\right)+y\left(1-\rho\right)$, the expression on the right-hand side
of the inequality above becomes
\begin{align*}
&= \frac{1}{2 \pi \sqrt{1-\rho^2}} \int_{[y-t,y+t]} \exp \left(-\frac{x_2^2}{2} \right)dx_2  \left(1+\rho\right)  \int_{[-t,t]} 
\exp \left(-\frac{\left(x_3\left(1+\rho\right)+y\left(1-\rho\right)\right)^2}{2\left(1-\rho^2\right)} \right)dx _3,
\end{align*}

\begin{align*}
& = \exp\left(-\frac{y^2\left(1-\rho\right)}{2\left(1+\rho\right)}\right) \frac{1}{2 \pi} 
\sqrt{\frac{1+\rho}{1-\rho}} \int_{[y-t,y+t]} \exp \left(-\frac{x_2^2}{2} \right)dx_2 \times \\
&\times  
\int_{[-t,t]} \exp \left(-\frac{x_3^2\left(1+\rho\right)^2+2x_3y\left(1-\rho^2\right)}{2\left(1-\rho^2\right)} \right)dx _3 \\
& \leq \exp\left(-\frac{y^2\left(1-\rho\right)}{2\left(1+\rho\right)}\right) \frac{1}{2 \pi} \sqrt{\frac{1+\rho}{1-\rho}} \int_{[y-t,y+t]} \exp \left(-\frac{x_2^2}{2} \right)dx_2  \int_{[-t,t]} \exp \left(-\frac{x_3^2}{2} + x_3y \right)dx _3\\
&=\exp\left( \frac{y^2\rho}{1+\rho} \right) \frac{1}{2 \pi} \sqrt{\frac{1+\rho}{1-\rho}} \int_{[y-t,y+t]} \exp \left(-\frac{x_2^2}{2} \right)dx_2  \int_{[-t,t]} \exp \left(-\frac{\left(x_3+y\right)^2}{2}\right)dx_3\\
&=\exp\left( \frac{y^2\rho}{1+\rho} \right) \frac{1}{2 \pi} \sqrt{\frac{1+\rho}{1-\rho}} \left(\int_{[y-t,y+t]} 
\exp \left(-\frac{x_2^2}{2} \right)dx_2\right)^2,
\end{align*} which is exactly:

\begin{align*}
\exp\left(\frac{y^2\rho}{1+\rho}\right)\sqrt{\frac{1+\rho}{1-\rho}} p_{t,y}^2  
\le \exp\left(y^2\rho\right)  \sqrt{\frac{1+\rho}{1-\rho}} p_{t,y}^2
\end{align*}
This completes the proof of Lemma~\ref{lemma:q-in-p}.
\end{proof}

\subsection{Roadmap of the Upper Bound's proof}

Recall, that our goal is to establish the required bound on the cardinality of the set (\ref{eq:SwitchNorm}) instead. Thus
for every $s>0$ we consider the counting random variable of interest,
\begin{align*}
Z_{s,\infty}=|\{ \beta \in \{0,1\}^p : \|\beta\|_0=k,||Y-X\beta||_{\infty}<s\}|.
\end{align*} Our goal is to establish that under our assumptions for sufficiently large constant $D_0>0$ and $s=D_0\sqrt{k+\sigma^2}\exp\left(-\frac{k \log p}{n}\right)$ it holds \begin{align}\label{GoalRoad} Z_{s,\infty} \geq D_0^{\frac{n}{3}}\end{align} w.h.p. as $k \rightarrow + \infty$.

To establish this we use a  conditional second moment method where the conditioning is happening on the ``target" vector $Y$.  We first show that the conditional first moment satisfies a similar property to (\ref{GoalRoad}); it holds \begin{align}\label{Step1Road}\E{Z_{s,\infty}|Y} \geq D_0^{\frac{n}{4}}\end{align} w.h.p. as $k \rightarrow + \infty$ (Lemma \ref{eq:Zt-zeta}). This step follows from standard algebraic manipulations and an appropriate use of the Law of Large Numbers.

 To establish (\ref{GoalRoad}) from (\ref{Step1Road}) we study the conditional second moment $\E{Z_{s,\infty}^2|Y}$ as well and specifically the ratio the squared first moment, \begin{align*}
\Upsilon=\Upsilon(Y)\triangleq 
\frac{\E{Z_{t\sqrt{k},\infty}^2|Y} }{\E{Z_{t\sqrt{k},\infty}|Y}^2},
\end{align*}where we have used for convenience $s=t\sqrt{k}$ for some $t$ which throughout the proof of order $O\left(1\right)$. The second moment analysis is done in two parts. The first part is an observation; if $\Upsilon(Y)$ converges to $1$ in expectation, then  (\ref{Step1Road}) implies (\ref{GoalRoad}). The proof of this part is based on the fact that for any probability measure and any positive random variable $R$ using Chebyshev's inequality,
\begin{align}\label{chebroad0}\mathbb{P}\left(R< \frac{\E{R}}{2}\right) \leq \mathbb{P}\left(|R-\E{R}|>\frac{\E{R}}{2}\right) \leq \frac{\E{R^2} }{\E{R}^2}-1.\end{align}We then  consider the conditional probability measure $\mathbb{P}$ on the random variable $Y$ for our setting and apply the above inequality for $R=Z_{t\sqrt{k},\infty}$ to derive,
\begin{align}\label{chebroad00} \mathbb{P}\left(Z_{t\sqrt{k},\infty} \geq \E{Z_{t\sqrt{k},\infty}|Y} |Y \right) \leq \Upsilon(Y)-1\end{align}
and therefore \begin{align}\label{chebroad} \mathbb{P}\left(Z_{t\sqrt{k},\infty} \geq \E{Z_{t\sqrt{k},\infty}|Y}\right) \leq \mathbb{E}_Y\{\Upsilon(Y)-1\}.\end{align}The first part follows immediately from (\ref{chebroad}).

Unfortunately we cannot establish that $\Upsilon=\Upsilon(Y)$ converges to $1$ in expectation due to a \textit{lottery effect};  it turns out that $\Upsilon$ can take arbitrary large values but with negligible probability which make the expected value of $\Upsilon$ to explode. The second part is to show that $\min\{\Upsilon,2\}$, the truncated version of $\Upsilon$, indeed converges to $1$ in expectation, as $k \rightarrow + \infty$. The exact statement of this part can be found in Proposition \ref{prop:Proposition1}. Note that the argument with the Chebyshev's inequality described above can be easily adapted to work for the truncated version of $\Upsilon$ simply because the probability on the right hand side of (\ref{chebroad00}) is upper bounded by $1$ allowing to improve (\ref{chebroad}) to 
\begin{align}\label{chebroad3} \mathbb{P}\left(Z_{t\sqrt{k},\infty} \geq \E{Z_{t\sqrt{k},\infty}|Y}\right) \leq \mathbb{E}_Y\{\min\{\Upsilon-1,1\}\}=\mathbb{E}_Y\{\min\{\Upsilon,2\}-1\}.\end{align}Establishing Proposition \ref{prop:Proposition1} comprises the bulk of the proof and requires the use various concentration of measure inequalities and properties of the (uni-variate and bi-variate) Gaussian density function.

\subsection{Conditional second moment bounds}

We start this subsection with obtaining estimates on  $\E{Z_{t \sqrt{k},\infty}|Y}$ and $\E{Z_{t \sqrt{k},\infty}^2|Y}$ for $t=O\left(1\right)$.

A direct calculation gives 
\begin{align*}
\E{ Z_{t\sqrt{k},\infty}|Y}= \binom{p}{k}
\prod_{i=1}^{ n} 
\mathbb{P}\left(|\frac{Y_i}{\sqrt{k}}-V|<t\right)=
 \binom{p}{k} \prod_{i=1}^{ n} p_{t,\frac{Y_i}{\sqrt{k}}},
\end{align*}
where $V$ is a standard normal random variable and $p_{t,y}$ was defined in (\ref{eq:pdef}). Similarly,
\begin{align*}
\E{Z_{t \sqrt{k},\infty}^2|Y}=\sum_{\ell=0}^{k}\binom{p}{k-\ell,k-\ell,\ell,p-2k+\ell}\prod_{i=1}^{n} 
\mathbb{P}\left(|Y_i-V_1^{\ell}|<t\sqrt{k},|Y_i-V_2^{\ell}|<t \sqrt{k}\right),
\end{align*} 
where   $V_1^{\ell},V_2^{\ell}$ are each $N\left(0,k\right)$ random variables with covariance  $l$. 
In terms of $q_{t,y,\rho}$  defined in (\ref{eq:qdef})  we have for every $l$, 
\begin{align*}
\mathbb{P}\left(|Y_i-V_1^{\ell}|<t\sqrt{k},|Y_i-V_2^{\ell}|<t \sqrt{k}\right)=q_{t,\frac{Y_i}{\sqrt{k}},\frac{\ell}{k}}.
\end{align*} 
Hence,
\begin{align*}
\E{Z_{t \sqrt{k+\sigma^2},\infty}^2|Y}=\sum_{\ell=0}^{k}\binom{p}{k-\ell,k-\ell,\ell,p-2k+\ell}   \prod_{i=1}^{ n} q_{t,\frac{Y_i}{\sqrt{k}},\frac{\ell}{k}}.
\end{align*} 
We obtain
\begin{align*}
\Upsilon=\Upsilon(Y)= \sum_{\ell=0}^{k}\frac{\binom{p}{k-\ell,k-\ell,\ell,p-2k+\ell}}{\binom{p}{k}^2}\prod_{i=1}^{n} 
\frac{q_{t,{Y_i\over \sqrt{k}},\frac{\ell}{k}}}{p_{t,{Y_i\over \sqrt{k}}}^2}.
\end{align*}
Now for $\ell=0$ and all $i=1,2,...,n$ we have $q_{t,{Y_i\over \sqrt{k}},0}=p_{t,{Y_i\over \sqrt{k}}}^2$ a.s. 
and therefore the first term of this sum equals  $\frac{\binom{p}{k,k,p-2k}}{\binom{p}{k}^2} \leq 1$.

We now analyze terms corresponding to $\ell \ge 1$. We have  for all $\ell=1,..,k$ 
$$
\binom{k}{\ell} \leq \frac{k^{\ell}}{\ell!} \leq k^{\ell},\binom{p-k}{k-\ell} \leq \frac{\left(p-k\right)^{k-\ell}}{\left(k-\ell\right)!}.
$$ 
By (\ref{eq:k-less-p}) we have $k^4\le p$ implying $k\leq \sqrt{p}$ and  applying Lemma~\ref{binomiallemma} we have
\begin{align*}
\binom{p}{k} \geq \frac{p^k}{4  k!}.
\end{align*} 
Combining the above we get that for every $\ell=1,...,k$  it holds:
\begin{align*}
\frac{\binom{p}{k-\ell,k-\ell,l,p-2k+\ell}}{\binom{p}{k}^2} = \binom{k}{\ell} \frac{\binom{p-k}{k-\ell}}{\binom{p}{k}} 
\leq k^{\ell}  \frac{\left(p-k\right)^{k-\ell}}{\left(k-\ell\right)!}   \frac{4 k!}{p^k} \leq 4 \left( \frac{p}{k^2}\right)^{-\ell}.
\end{align*}
Hence we have 
\begin{equation}
\Upsilon \leq 1+4\sum_{\ell=1}^{k} \left( \frac{p}{k^2}\right)^{-\ell}\prod_{i=1}^{n} \frac{q_{t,{Y_i\over \sqrt{k}},\frac{\ell}{k}}}{p_{t,{Y_i\over \sqrt{k}}}^2} .
\label{eq:uppboundA}
\end{equation} 
Our key result regarding the conditional second moment estimate and its ratio to the square of the conditional first moment
estimate is the following proposition.
\begin{proposition}\label{prop:Proposition1}
Suppose  $k\log k\le Cn$ for all $k$ and $n$ for some constant $C>0$. 
Then  for all sufficiently large constants $D>0$ there exists $c>0$ such that for 
$n\le \frac{k \log \left(\frac{p}{k^2}\right)}{ 2\log D}$ 
and $t=D\sqrt{1+\sigma^2}  \left(\frac{p}{k^2}\right)^{-\frac{k}{n}}$ we have
\begin{align*}
\mathbb{E}_Y\left(\min \{1,\Upsilon-1 \}\right) \leq {1\over k^c}.
\end{align*} 
\end{proposition}

\begin{proof}
Fix a parameter $\zeta \in \left(0,1\right)$ which will be optimized later. We have,
 \begin{align*}
\mathbb{E}_Y\left(\min \{1,\Upsilon-1 \}\right)&
=\mathbb{E}_Y\left(\min \{1,\Upsilon-1 \}{\bf 1}\left(\min \{1,\Upsilon-1 \} \geq \zeta^n\right)\right)\\
&+\mathbb{E}_Y\left(\min \{1,\Upsilon-1 \}{\bf 1}\left(\min \{1,\Upsilon-1 \}\leq \zeta^n\right)\right) \\
&\leq \mathbb{P}\left(\min \{1,\Upsilon-1 \} \geq \zeta^n\right)+\zeta^n.
\end{align*}  
Observe that if $ \Upsilon \geq 1+\zeta^n$, then (\ref{eq:uppboundA}) implies that at least one of the summands of 
\begin{align*}
\sum_{\ell=1}^{k}  4\left(\frac{p}{k^2}\right)^{-\ell}\prod_{i=1}^{n} \frac{q_{t,{Y_i\over \sqrt{k}},\ell}}{p_{t,{Y_i\over \sqrt{k}}}^2}
\end{align*} for $\ell=1,2..,k$ should be at least  $\frac{\zeta^n}{k}$. Hence applying the  union bound,
\begin{align*}
\mathbb{P}\left(\min \{1,\Upsilon-1 \} \geq \zeta^n\right)
&\leq \mathbb{P}\left(\Upsilon \geq 1+\zeta^n\right)\\
&\leq \mathbb{P}\left(\bigcup_{\ell=1}^{k} \{4 \left(\frac{p}{k^2}\right)^{-\ell}\prod_{i=1}^{n} 
\frac{q_{t,{Y_i\over \sqrt{k}},{\ell\over k}}}{p_{t,{Y_i\over \sqrt{k}}}^2}\ge\frac{\zeta^n}{k} \}\right) \\
&\leq \sum_{\ell=1}^{k} \mathbb{P}\left(4 \left(\frac{p}{k^2}\right)^{-\ell}\prod_{i=1}^{n} \frac{q_{t,{Y_i\over \sqrt{k}},{\ell\over k}}}{p_{t,{Y_i\over \sqrt{k}}}^2}\ge\frac{\zeta^n}{k}\right)
\end{align*} 
Introducing parameter $\rho=\frac{\ell}{k}$ we obtain
\begin{equation}
\mathbb{E}_Y\left(\min \{1,\Upsilon-1 \}\right)  \leq \zeta^n+\mathbb{P}\left(\min \{1,\Upsilon-1 \} 
\geq \zeta^n\right) \leq \zeta^n+\sum_{\rho=\frac{1}{k},\frac{2}{k},..,1} \mathbb{P}\left(\Upsilon_{\rho}\right),
\label{eq:targetexp}
\end{equation} where for all $\rho=\frac{1}{k},..,\frac{k-1}{k},\frac{k}{k}$ we define 
$$\Upsilon_{\rho}\triangleq
{\Big \{}4 \left(\frac{p}{k^2}\right)^{-\rho k}\prod_{i=1}^{n} \frac{q_{t,{Y_i\over \sqrt{k}},\rho }}{p_{t,{Y_i\over \sqrt{k}}}^2}\ge \frac{\zeta^n}{k} {\Big\}}.
$$ 

Next we obtain an upper bound on $\mathbb{P}\left(\Upsilon_{\rho}\right)$ for any $\rho \in (0,1]$ as a function of $\zeta$.
Set 
\begin{align*}
\rho_*:=1- \frac{n \log D}{3k \log (p/k^2)}.
\end{align*} 
The cases 
 $\rho \leq \rho_*$ and  $\rho>\rho_*$ will be considered separately.
\begin{lemma}\label{lemma:rho-large}
For  all $\rho \in (\rho_*,1]$ and $\zeta\in (0,1)$.
\begin{align*} 
\mathbb{P}\left(\Upsilon_{\rho}\right) \leq 2^n\left(D^{-{1\over 18}}\zeta^{-{1\over 6}}\right)^n.
\end{align*} 
\end{lemma}

\begin{proof}
Since $\rho>\rho_{*}$  then
\begin{equation}
-\left(1-\rho\right) \frac{k \log \left(\frac{p}{k^2}\right)}{n} \geq - \frac{1}{3}\log D.
\label{eq:assump1}
\end{equation} 
Now we have $ q_{t,{Y_i\over \sqrt{k}},\rho} \leq p_{t,{Y_i\over \sqrt{k}}}$ which implies 
$\frac{q_{t,{Y_i\over \sqrt{k}},\rho }}{p_{t,{Y_i\over \sqrt{k}}}^2} \leq p_{t,{Y_i\over \sqrt{k}}}^{-1}$, which after taking logarithms and dividing both the sides by $n$ gives
\begin{align*} 
\mathbb{P}\left(\Upsilon_{\rho}\right)
&\leq \mathbb{P}\left(\frac{1}{n} \sum_{i=1}^{n} -\log p_{t,{Y_i\over \sqrt{k}}} \geq \log \zeta - \frac{\log 4 k}{n} +\rho \frac{k \log \frac{p}{k^2}}{n}\right).
\end{align*}
Applying (\ref{eq:bound-log-p}) we obtain
\begin{align*}
\mathbb{P}\left(\Upsilon_{\rho}\right)
&\leq \mathbb{P}\left(-\log t+{t^2\over 2}+{1\over n}\sum_{i=1}^n{Y_i^2\over 2k}+(1/2)\log(2/\pi) 
\geq \log \zeta - \frac{\log 4 k}{n} +\rho \frac{k \log \frac{p}{k^2}}{n}\right),
\end{align*}
Recall that $t=D\sqrt{1+\sigma^2}  \left(\frac{p}{k^2}\right)^{-\frac{k}{n}}$, namely 
$\log t\ge \log D-{k\over n}\log\left(\frac{p}{k^2}\right)$ and thus applying (\ref{eq:assump1})
\begin{align*}
\log t+\rho \frac{k \log \frac{p}{k^2}}{n}&\ge -(1-\rho)\frac{k \log \frac{p}{k^2}}{n}+\log D \\
&\ge {2\over 3}\log D.
\end{align*}
By the bound on $n$, we have
$t\le D\sqrt{1+\sigma^2}/D^2\le 2/D\le 1$ for sufficiently large $D$. The same applies to $t^2/2$.
Also since $k\log k\le Cn$ then $\log(4k)/n\le C/k+\log 4/(k\log k)$. 
Then for sufficiently large $D$ we obtain
\begin{align*}
\mathbb{P}\left(\Upsilon_{\rho}\right) &\leq \mathbb{P}\left({1\over n}\sum_{i=1}^n{Y_i^2\over 2k}\geq \log \zeta+(1/3)\log D\right) \\
&=\mathbb{P}\left(\exp\left({1\over 6}\sum_{i=1}^n{Y_i^2\over 2k}\right)\geq \zeta^{n\over 6}D^{n\over 18}\right)\\
&\le {1\over \zeta^{n\over 6}D^{n\over 18}}\left(\E{\exp\left({Y_1^2\over 12k}\right)}\right)^n
\end{align*}
Recall that since $Y_1$ has distribution $N(0,\sigma^2)$ and $\sigma^2\le 3k$ then
\begin{align*}
\E{\exp\left({Y_1^2\over 12k}\right)}={1\over \sqrt{1-2\sigma^2/(12k)}}\le \sqrt{2}.
\end{align*}
We obtain a bound
\begin{align*}
\mathbb{P}\left(\Upsilon_{\rho}\right) &\leq 2^n\left(D^{-{1\over 18}}\zeta^{-{1\over 6}}\right)^n,
\end{align*}
as claimed. 
\end{proof}

\begin{lemma}\label{lemma:rho-less-rho-star} 
For  all $\rho \in [{1\over k},\rho_*]$ and $\zeta\in (0,1)$. 
\begin{align*}
\mathbb{P}\left(\Upsilon_{\rho}\right) \leq 4^n\left(D^{1\over 2}\zeta^k \right)^{-n/12}.
\end{align*}
\end{lemma}

\begin{proof}
Applying Lemma~\ref{lemma:q-in-p} we have
\begin{align*}
\mathbb{P}\left(\Upsilon_{\rho}\right)&=\mathbb{P}
\left(4 \left(\frac{p}{k^2}\right)^{-\rho k}\prod_{i=1}^{n} \frac{q_{t,{Y_i\over \sqrt{k}},\rho }}{p_{t,{Y_i\over \sqrt{k}}}^2}\ge \frac{\zeta^n}{k}\right)\\
&\leq \mathbb{P}\left(4 \left(\frac{p}{k^2}\right)^{-\rho k}\prod_{i=1}^{n} 
\left( \sqrt{\frac{1+\rho }{1-\rho}} \exp\left( \rho {Y_i^2\over k}\right) \right)\ge \frac{\zeta^n}{k} \right) \\
&= \mathbb{P}\left( \rho \sum_{i=1}^n {Y_i^2\over kn} \geq \log \zeta -\frac{ \log 4 k}{n}+ \frac{1}{2} 
\log \left( \frac{1-\rho}{1+\rho}\right)+ \frac{\rho k  \log \left(\frac{p}{k^2}\right)}{n} \right) \\
&=\mathbb{P}\left( \sum_{i=1}^n {Y_i^2\over kn} \geq \rho^{-1}\log \zeta -\rho^{-1}\frac{ \log 4 k}{n}+ 
\frac{1}{2\rho} \log \left( \frac{1-\rho}{1+\rho}\right)+ \frac{ k  \log \left(\frac{p}{k^2}\right)}{n} \right).
\end{align*}
Let 
\begin{align*}
f\left(\rho\right)=\rho^{-1}\log \zeta -\rho^{-1}\frac{ \log 4 k}{n}+ 
\frac{1}{2\rho} \log \left( \frac{1-\rho}{1+\rho}\right)+ \frac{ k  \log \left(\frac{p}{k^2}\right)}{n}. 
\end{align*}
Applying Lemma~\ref{concavelemma} and that $\zeta<1$ we can see that the function $f$ is concave. 
This implies that the minimum value of $f$ for $\rho \in [\frac{1}{k},\rho_*]$ is either $f\left(\frac{1}{k}\right)$ or $f\left(\rho_*\right)$,
and therefore
\begin{align}
\mathbb{P}\left(\Upsilon_{\rho}\right) &\leq \mathbb{P}\left( \sum_{i=1}^{n} {Y_i^2\over kn} \geq \min \{ f\left(\frac{1}{k}\right),f\left(\rho_*\right) \}\right) \notag\\
& \leq \mathbb{P}\left( \sum_{i=1}^{n} {Y_i^2\over kn} \geq f\left(\frac{1}{k}\right)\right)+ \mathbb{P}\left( \sum_{i=1}^{n} {Y_i^2\over kn} \geq f\left(\rho_*\right)\right).  \label{eq:rhooo}
\end{align}

Now we apply a standard Chernoff type bound on  
$\mathbb{P}\left( \sum_{i=1}^{n} {Y_i^2\over k} \geq n w\right)$ for $w \in \mathbb{R}$. 
We have
$\E{\exp \left(\theta Y_i^2/k\right) }=\frac{1}{\sqrt{1-2(\sigma^2/k) \theta}}<\infty$ if $\theta<\frac{1}{2\sigma^2/k}$. 
Since in our case $1 \leq \E{Y_i^2\over k}=\sigma^2/k \leq 3$, to obtain a finite bound we set $\theta=\frac{1}{12}<\frac{1}{6}$ and obtain 
\begin{align*}
\E{\exp\left({Y_i^2\over 12k}\right)}
=\frac{1}{\sqrt{1-\frac{\sigma^2}{6k}}} \leq \sqrt{2}.
\end{align*}
Therefore, we obtain
\begin{align*}
\mathbb{P}\left( \sum_{i=1}^{n} {Y_i^2\over k} \geq n w\right) &\leq 
\exp\left(-n\frac{w}{12}\right) 
\left(\E{\exp\left({Y_i^2\over 12k}\right)}\right)^n \\
&\leq 2^{n\over 2}\exp(-nw/12) .
\end{align*}
We obtain 
\begin{equation}
\mathbb{P}\left(\Upsilon_{\rho}\right) \leq 
2^{n\over 2}\exp(-nf(1/k)/12) + 2^{n\over 2}\exp(-nf(\rho^*)/12).
\label{eq:bound1}
\end{equation}
Now we obtain bounds on  $f\left(\frac{1}{k}\right)$ and $f\left(\rho_*\right)$. We have
\begin{align*}
f\left({1\over k}\right)=k\log \zeta -\frac{ k\log 4 k}{n}+ 
\frac{k}{2} \log \left( \frac{1-{1\over k}}{1+{1\over k}}\right)+ \frac{ k  \log \left(\frac{p}{k^2}\right)}{n}. 
\end{align*}
We have by our assumption $k\log k\le Cn$ that $k\log (4k)/n \le Ck\log(4k)/(k\log k)$. The sequence
$\frac{k}{2} \log \left( \frac{1-{1\over k}}{1+{1\over k}}\right)$ is bounded by a universal constant for $k\ge 2$.
Finally, we have $n\le k\log(p/k^2)/(2\log D)$. Thus for sufficiently large $D$,
\begin{align*}
f\left({1\over k}\right)\ge k\log\zeta +\log D,
\end{align*}
implying
\begin{align*}
2^{n\over 2}\exp(-nf(1/k)/12)\le 2^{n\over 2}\left(D\zeta^k\right)^{-n/12}.
\end{align*}
Now we will bound $f\left(\rho_*\right)$. We have
\begin{align*}
f\left(\rho^*\right)=(1/\rho^*)\log \zeta -(1/\rho^*)\frac{ \log 4 k}{n}+ 
\frac{1}{2\rho^*} \log \left( \frac{1-\rho^*}{1+\rho^*}\right)+ \frac{ k  \log \left(\frac{p}{k^2}\right)}{n}. 
\end{align*}
Applying upper bound on $n$, we have  $\rho_*>1/2$. Then $-1/(2\rho^*)\log(1+\rho^*)\ge -\log 2$. We obtain
\begin{align*}
f\left(\rho^*\right)=2\log \zeta -2\frac{ \log 4 k}{n}+ 
\log \left(1-\rho^*\right)+ \frac{ k  \log \left(\frac{p}{k^2}\right)}{n}. 
\end{align*}
We have again 
\begin{align}\label{eq:bound-log-4k}
2\log(4k)/n\le 2C\log(4k/k).
\end{align}
Applying the value of $\rho^*$ we have
\begin{align*}
\log \left(1-\rho^*\right)+ \frac{ k  \log \left(\frac{p}{k^2}\right)}{n}=
-\log\left({3k\log (p/k^2)\over n\log D}\right)+ \frac{ k  \log \left(\frac{p}{k^2}\right)}{n}.
\end{align*}
Consider 
\begin{align*}
-\log\left({3k\log (p/k^2)\over \log D}\right)+\log n+\frac{ k  \log \left(\frac{p}{k^2}\right)}{n}.
\end{align*}
For every $a>0$, the function $\log x+a/x$ is a decreasing on $x\in (0,a]$ and thus, applying the bound $n\le k\log(p/k^2)/(2\log D)$,  
the expression above is at least
\begin{align*}
-\log\left({3k\log (p/k^2)\over \log D}\right)+\log \left(k\log(p/k^2)/(2\log D)\right)+2\log D 
&=-\log 3-\log 2+2\log D\\
&\ge (3/2)\log D,
\end{align*}
for sufficiently large $D$. Combining with (\ref{eq:bound-log-4k}) we obtain that for sufficiently large $D$
\begin{align*}
f(\rho^*)\ge 2\log \zeta+\log D,
\end{align*}
Combining two bounds we obtain
\begin{align*}
\mathbb{P}\left(\Upsilon_{\rho}\right) &\leq 
2^{n\over 2}\left(D\zeta^k\right)^{-{n\over 12}}+
2^{n\over 2}\left(D\zeta^2\right)^{-n/12} \\
&\le 2^{{n\over 2}+1}\left(D\zeta^k\right)^{-{n\over 12}}.
\end{align*}
\end{proof}

We now return to the proof of Proposition~\ref{prop:Proposition1}.
Combining the results of Lemma~\ref{lemma:rho-large}
and Lemma~\ref{lemma:rho-less-rho-star}, and assuming $k\ge 6\cdot 12=72$, we obtain that 
\begin{align*}
\mathbb{P}\left(\Upsilon_{\rho}\right) &\leq 2^n\left(D^{1\over 18}\zeta^6\right)^{-n}+2^{{n\over 2}+1}\left(D\zeta^k \right)^{-n/12} \\
&\le 2^{n+1}\left(D^{1\over 2}\zeta^k \right)^{-n/12}
\end{align*}
for all $\rho\in [1/k,1]$ and $\zeta\in (0,1)$. Recalling (\ref{eq:targetexp})  we obtain
\begin{align*}
\mathbb{E}_Y\left(\min \{1,\Upsilon-1 \}\right)  &\leq \zeta^n+(2k)2^{n}\left(D^{1\over 2}\zeta^k \right)^{-n/12}.
\end{align*}
Let $D_1\triangleq D^{1\over 2}/2^{12}$ and rewrite the bound above as
\begin{align*}
\zeta^n+(2k)\left(D_1\zeta^k \right)^{-n/12}.
\end{align*}
Assume $D$ is large enough so that $D_1>1$ and let $\zeta=1/D_1^{1\over 2k}<1$. We obtain a bound
\begin{align*}
D_1^{-{n\over 2k}}+(2k)D_1^{-n/24}.
\end{align*}
Finally since $n\ge (1/C)k\log k$, we obtain a bound of the form $1/k^c$ for some constant $c>0$ as claimed.
This completes the proof of Proposition~\ref{prop:Proposition1}.

\end{proof}


\subsection{The Upper Bound}

\begin{proof}[Proof  of Theorem~\ref{theorem:PureNoiseLowerBound}]
By an assumption of the theorem, we have $k^4\le p$. Thus 
\begin{align*}
k\log p \le 2k\log (p/k^2).
\end{align*}
Then 
\begin{align}
n&\le {k\log p\over 2\log D_0} 
\le {k\log(p/k^2)\over \log D_0}
={k\log(p/k^2)\over 2\log D_0^{1\over 2}}. \label{eq:Bound-on-n}
\end{align}
Our goal is to obtain a lower bound on the cardinality of the set
\begin{align*}
{\Big\{}\beta \in \{0,1\}^p: \|\beta\|_0=k, \|Y-X\beta\|_\infty\le D_0\sqrt{k}\sqrt{1+\sigma^2/k}\exp\left(-{k\log p\over n}\right){\Big\}},
\end{align*}
Recall that $k\le \sigma^2\le 3k$. Letting
\begin{align*}
t_0=D_0\sqrt{1+\sigma^2/k}\exp\left(-{k\log p\over n}\right),
\end{align*}
our goal is then obtaining a lower bound on  $Z_{t_0\sqrt{k}}$.
Since $k\log k\le Cn$, then for sufficiently large $D_0$, 
\begin{align*}
t_0\ge D_0^{1\over 2}\sqrt{1+\sigma^2/k}\exp\left(-{k\log (p/k^2)\over n}\right)\triangleq \tau,
\end{align*}
and thus it suffices to obtain the claimed bound on $Z_{t_1\sqrt{k}}$. We note that by our bound (\ref{eq:Bound-on-n})
\begin{align}\label{eq:bound-on-tau}
\tau\le D_0^{1\over 2}\sqrt{1+\sigma^2/k}/D_0\le 2/D_0^{1\over 2}\le 1,
\end{align}
provided $D_0$ is sufficiently large.
Let $D=D_0^{1\over 2}$. Then, by the
definition of $\tau$ and by (\ref{eq:Bound-on-n}) the assumptions of  Proposition~\ref{prop:Proposition1}
are satisfied for this choice of $D$ and $t=\tau$.

\begin{lemma}\label{eq:Zt-zeta}
The following bound holds
with high probability with respect to $Y$ as $k$ increases
\begin{align*}
n^{-1}\log \E {Z_{\tau\sqrt{k} , \infty} |Y } \ge (1/2)\log D.
\end{align*}
\end{lemma}
\begin{proof}
As before for $Y=(Y_1,\ldots,Y_n)$,
\begin{align*}
\E{ Z_{\tau\sqrt{k},\infty}|Y}= \binom{p}{k} \prod_{i=1}^{ n} \mathbb{P}\left(|\frac{Y_i}{\sqrt{k}}-X|<t|Y\right)=\binom{p}{k} \prod_{i=1}^{ n} p_{\tau,\frac{Y_i}{\sqrt{k}}},
\end{align*}
where $X$ is the standard normal random variable. 
Taking logarithms,
\begin{equation}
\log \E{ Z_{\tau\sqrt{k},\infty}|Y} = \log \binom{p}{k} +\sum_{i=1}^{n} \log p_{\tau,\frac{Y_i}{\sqrt{k}}}.
\label{eq:lem0}
\end{equation} 
Applying (\ref{eq:bound-log-p}), we have
\begin{align*}
n^{-1}\log \E{ Z_{\tau\sqrt{k},\infty}|Y}& \ge n^{-1}\log \binom{p}{k} +\log \tau-{\tau^2\over 2}+(1/2)\log(2/\pi)
-n^{-1}\sum_{i=1}^{n}{Y_i^2\over 2k}
\end{align*}
Using
\begin{align*}
\tau\ge D\exp\left(-{k\log(p/k^2)\over n}\right),
\end{align*}
and $\tau\le 1$, we obtain
\begin{align*}
n^{-1}\log \E{ Z_{\tau\sqrt{k},\infty}|Y}& \ge n^{-1}\log \binom{p}{k} +\log D-{k\log(p/k^2)\over n}-{1\over 2}+(1/2)\log(2/\pi)
-n^{-1}\sum_{i=1}^{n}{Y_i^2\over 2k}
\end{align*}
Since by (\ref{eq:k-less-p}) we have $k\le \sqrt{p}$, applying
Lemma~\ref{binomiallemma}  we have $\frac{1}{n}\log \binom{p}{k}-\frac{k}{n} \log \left(\frac{p}{k^2}\right)  \geq 0$.
By Law of Large Numbers and since $Y_i$ is distributed as $N(0,\sigma^2)$ with  $k\le \sigma^2\le 3k$, we
have $n^{-1}\sum_{i=1}^{n}{Y_i^2\over 2k}$ converges to $\sigma^2/(2k)\le 3/2$ as $k$ and therefore $n$ increases. Assuming $D$ 
is sufficiently large we obtain that w.h.p. as $k$ increases,
\begin{align*}
n^{-1}\log \E{ Z_{\tau\sqrt{k},\infty}|Y}& \ge (1/2)\log D.
\end{align*}
This concludes the proof of the lemma.
\end{proof}

Now we claim   that w.h.p. as $k$ increases,
\begin{align}\label{eq:Z-large-half-expectation}
Z_{\tau\sqrt{k},\infty}\ge\frac{1}{2} \E{Z_{\tau\sqrt{k},\infty}|Y}.
\end{align}

We have 
\begin{equation}
\mathbb{P}\left(Z_{\tau\sqrt{k},\infty}<\frac{1}{2} \E{Z_{\tau\sqrt{k},\infty}|Y}\right)
\leq \mathbb{P}\left(|Z_{\tau\sqrt{k},\infty}-\E{Z_{\tau\sqrt{k},\infty}|Y} | \geq \frac{1}{2}\E{Z_{\tau\sqrt{k},\infty}|Y}\right),
\label{eq:easystat}
\end{equation} 
and applying Chebyshev's inequality we obtain,
$$\mathbb{P}\left(|Z_{\tau\sqrt{k},\infty}-\E{Z_{\tau\sqrt{k},\infty}|Y} | \geq \frac{1}{2}\E{Z_{\tau\sqrt{k},\infty}|Y} |Y\right) 
\leq 4 \min \left[ \frac{\E{Z^2_{\tau\sqrt{k},\infty}|Y}}{\E{Z_{\tau\sqrt{k},\infty}|Y}^2}-1 ,1 \right].$$
Hence, taking expectation over $Y$ we obtain,
\begin{equation}
\mathbb{P}\left(|Z_{\tau\sqrt{k},\infty}-\E{Z_{\tau\sqrt{k},\infty}|Y} | \geq \frac{1}{2}\E{Z_{\tau\sqrt{k},\infty}|Y}\right)
\leq 4 \mathbb{E}_Y\left[ \min \left[ \frac{\E{Z^2_{\tau\sqrt{k},\infty}|Y}}{\E{Z_{\tau\sqrt{k},\infty}|Y}^2}-1 ,1 \right]\right]. \notag
\end{equation} 
We conclude  
\begin{equation}
\mathbb{P}\left(Z_{\tau\sqrt{k},\infty}<\frac{1}{2} \E{Z_{\tau\sqrt{k},\infty}|Y}\right)
\leq  4 \mathbb{E}_Y\left[ \min \left[ \frac{\E{Z^2_{\tau\sqrt{k},\infty}|Y}}{\E{Z_{\tau\sqrt{k},\infty}|Y}^2}-1 ,1 \right]\right].
\end{equation}
Applying Proposition~\ref{prop:Proposition1} the assumptions of which have been verified as discussed above, we obtain
\begin{align*}
\mathbb{P}\left(Z_{\tau\sqrt{k},\infty}<\frac{1}{2} \E{Z_{\tau\sqrt{k},\infty}|Y}\right)
&\leq\E{\min\{1,\Upsilon-1\}|Y} \\
&\le k^{-c},
\end{align*}
for some $c>0$. This establishes the claim (\ref{eq:Z-large-half-expectation}).
Combining with Lemma~\ref{eq:Zt-zeta}, we conclude that w.h.p. as $k$ increases
\begin{align*}
n^{-1}\log Z_{\tau\sqrt{k},\infty}&\ge n^{-1}\log \E{Z_{\tau\sqrt{k},\infty}|Y}-\log 2/n \\
&\ge (1/2)\log D-\log 2/n.
\end{align*}
Since $n$ satisfying $Cn\ge k\log k$ increases as $k$ increases, we conclude that w.h.p. as $k$ increases
$Z_{\tau\sqrt{k},\infty}\ge D^{n\over 3}$.
This concludes the proof of the theorem.
\end{proof}

\section{Proof of Theorem~\ref{theorem:MainResult1}}\label{section:Proof of MainResult1}
In this section we  prove Theorem~\ref{theorem:MainResult1}.
The proof is based on a reduction scheme to the simpler optimization problem $\Psi_2$ which is analyzed in the previous section.

To prove Theorem ~\ref{theorem:MainResult1} we will also consider the following restriction of $\Phi_2$. For any $S \subseteq \mathrm{Support}\left(\beta^*\right)$ consider the optimization problem
$\left(\Phi_{2}\left(S\right)\right)$: 
$$
\begin{array}{clc} \left(\Phi_{2}\left(S\right)\right) & \min  &n^{-\frac{1}{2}}||Y-X\beta||_{2} \\ &\text{s.t.}&\beta \in \{0,1\}^p  \\
&& ||\beta||_0=k, \mathrm{Support}\left(\beta\right) \cap \mathrm{Support}\left(\beta^*\right) = S,
\end{array}
$$ 
and set $\phi_{2}\left(S\right)$ its optimal value. Notice that for a binary $k$-sparse $\beta$ with $ \mathrm{Support}\left(\beta\right) \cap \mathrm{Support}\left(\beta^*\right) = S$ we have:
\begin{align*}
&Y-X\beta=X\beta^*+W-X\beta\\
&=\sum_{i \in \mathrm{Support}\left(\beta^*\right)} X_i+W-\sum_{i \in \mathrm{Support}\left(\beta\right)}X_i\\
&=\sum_{i \in \mathrm{Support}\left(\beta^*\right)-S} X_i+W-\sum_{i \in \mathrm{Supp}\left(\beta\right)-S}X_i\\
&=Y'-X'\beta_1,
\end{align*} where we have defined $Y',X',\beta_1$ as following:

\begin{enumerate}
\item $X' \in \mathbb{R}^{n \times \left(p-k\right)}$ to be the matrix which is $X$ after deleting the columns corresponding to 
$\mathrm{Support(\beta^*)}$

\item $Y':=\sum_{i \in \mathrm{Support(\beta^*})-S} X_i+W$

\item $\beta_1\in \{0,1\}^{p-k}$ is obtained from $\beta$ after deleting coordinates in $\mathrm{Support(\beta^*)}$.
 Notice that $||\beta_1||_0=k-|S|$.
\end{enumerate}  

Hence, solving $\Phi_2\left(S\right)$ can be written equivalently with respect to $Y',X',\beta'$ as following,  $$\begin{array}{clc} \left(\Phi_{2}\left(S\right)\right) & \min  &n^{-\frac{1}{2}}||Y'-X'\beta'||_{2} \\ &\text{s.t.}&\beta' \in \{0,1\}^{p-k}  \\
&& ||\beta'||_0=k-|S|.

\end{array}$$

We claim that the above problem is satisfying all the assumptions of Theorem~\ref{theorem:PureNoiseLowerBound} except for one of the
assumptions which we discuss below. 
Indeed, $Y',X'$ are independent since they are functions of disjoint parts of $X$, $X'$ has standard Gaussian i.i.d. elements,  
$Y'=\sum_{i \in \mathrm{Support(\beta^*})-S} X_i+W$ has iid Gaussian elements with zero mean and variance $\left(k-|S|\right)+\sigma^2 $,
and the sparsity  of $\beta'$ is $k-|S|$. The only difference is  that the ratio between the 
variance $\left(k-|S|\right)+\sigma^2 $ and the sparsity  $k-|S|$ is no longer necessarily upper bounded by 3, since this holds if and only if 
$\sigma^2 \leq 2\left(k-|S|\right)$, which does not hold necessarily, though it does hold in the special
case $S=\emptyset$, provided $\sigma^2\le 2k$.
Despite the absence of this assumption for general $S$ 
we can still apply the lower bound (\ref{eq:FirstMomentBound}) of Theorem~\ref{theorem:PureNoiseLowerBound},
since the restriction on the relative value of the standard deviation of $Y_i$ and other restrictions on $p,n,k$ were needed only
for the upper bound. 
Hence, applying the first part of  Theorem~\ref{theorem:PureNoiseLowerBound} we  conclude the optimal value 
$\phi_2\left(S\right)$ satisfies

\begin{align}\label{eq:PhiS}
\mathbb{P}&\left(\phi_2\left(S\right)\ge 
e^{-{3\over 2}}\sqrt{2\left(k-|S|\right)+\sigma^2}
\exp\left(-\frac{\left(k-|S|\right)\log\left(\left(p-k\right)\right)}{n}\right)\right) \notag\\
&\ge 1-\exp(-n).
\end{align}

Also applying the second part of this theorem to  the special case $S=\emptyset$ we obtain the following corollary
for the case $\sigma^2\le 2k$.

\begin{coro}\label{coro:Case-2k}
Suppose $\sigma^2\le 2k$.
For every $C>0$ and every sufficiently large constant $D_0$,  if $k\log k\le Cn$,
 and $n\le k\log (p-k)/(2\log D_0)$, the cardinality of the set
\begin{align*}
{\Big\{}\beta \in \{0,1\}^p: \|\beta\|_0=k, n^{-{1\over 2}}\|Y'-X'\beta\|_2\le D_0\sqrt{2k+\sigma^2}\exp\left(-{k\log (p-k)\over n}\right){\Big\}}
\end{align*}
is at least $D_0^{n\over 3}$ w.h.p. as $k\rightarrow\infty$. 
\end{coro}

\begin{proof}[Proof of Theorem~\ref{theorem:MainResult1}]
Applying the union bound and (\ref{eq:PhiS}) we obtain 
\begin{align*}
\mathbb{P}&\left(\phi_2\left(\ell\right)\ge 
e^{-{3\over 2}}\sqrt{2\ell+\sigma^2}
\exp\left(-\frac{\ell\log\left(p-k\right)}{n}\right),~\forall~0\le \ell\le k\right) \notag\\
&\ge 1-\sum_{0\le \ell\le k}{k\choose \ell}\exp(-n) \\
&\ge 1-2^k\exp(-n).
\end{align*}
Since $k\log k\le Cn$, we have $2^k\exp(-n)\rightarrow 0$ as $k$ increases. Replacing
$p-k$ by a larger value $p$ in the exponent we complete the proof of part (a) of the theorem.

We now establish the second part of the theorem. It follows almost immediately from Corollary~\ref{coro:Case-2k}.
Since $k\log k\le Cn$, the bound $n\le k\log p/(3\log D_0)$ implies $\log k\le C\log p/(3\log D_0)$ and in particular
$k\log(p-k)=k\log p-O(\frac{k^2}{ p})$ and $k^2\over p$ converges to zero as $k$ increases, provided $D_0$ is sufficiently large.
Then we obtain
$n\le \exp(-k\log (p-k)/(2\log 2D_0))$
for all sufficiently large $k$. By a similar reason we may now replace $\exp(-k\log (p-k))$ by $\exp(-k\log p)$ in the 
upper bound on $n^{-{1\over 2}}\|Y'-X'\beta\|_2$ using the extra factor $2$ in front of $D_0$. This completes
the proof of the second part of the theorem.
\end{proof}

\section{The optimization problem $\Phi_2$}\label{section:Problem_Phi2}

In this section we give proofs of Proposition~\ref{prop:GammaMonotonic} and Theorem~\ref{theorem:sharptheorem}.

\begin{proof}[Proof of Proposition~\ref{prop:GammaMonotonic}]

It is enough to study $f=\log \Gamma$ with respect to monotonicity. We compute the derivative for every 
$\zeta \in [0,1]$, 
\begin{align*}
f'\left(\zeta\right)=-\frac{k \log p}{n}+\frac{k}{2\zeta k+\sigma^2}=-\frac{k}{n\left(2\zeta k+\sigma^2\right)} 
\left( \log p \left(2 \zeta k+\sigma^2\right)-n\right).
\end{align*}
Clearly, $f'$ is strictly decreasing in $\zeta$ and
$f'\left(\zeta\right)=0$ has a unique solution $\zeta^*=\frac{1}{2k\log p}\left(n-\sigma^2 \log p\right)$.
Using the strictly decreasing property of $f'$ and the fact that it has a unique root, 
we conclude that for $\zeta<\zeta^*$, $f'\left(\zeta\right)>0$, and for $\zeta>\zeta^*$, $f'\left(\zeta\right)<0$. 
As a result, if $\zeta^*\le 0$ then $f$ is a decreasing function on $[0,1]$, if
$\zeta^*\ge 1$ $f$ is an increasing function on $[0,1]$, and if  $\zeta^* \in \left(0,1\right)$ then $f$ is non monotonic.
These cases are translated to the cases $n\le \sigma^{2} \log p$, $n\ge (2k+\sigma^{2})\log p$ and $n\in \left(\sigma^{2}\log p,(2k+\sigma^{2})\log p\right)$,
respectively. The minimum value achieved by $f$, and its dependence on $n^{*}$ was already established earlier.
\end{proof}


\begin{proof}[Proof of Theorem~\ref{theorem:sharptheorem}]
 We set 
\begin{align*}
\Lambda_p\triangleq \mathrm{argmin}_{\ell=0,1,..,k} \phi_2\left(\ell\right),
\end{align*} and we remind the reader that
$\mathrm{argmin}_{\ell=0,1,..,k} \phi_2\left(\ell\right)=k-|\mathrm{Support}\left(\beta_2\right) \cap \mathrm{Support}\left(\beta^*\right)|$.

\vspace{.1in}
\textbf{ Case 1}: $n>\left(1+\epsilon\right)n^*$.
Showing $\|\beta_2-\beta^*\|_0/k\rightarrow 0$ as $k$ increases is equivalent to showing  
\begin{align*}
\frac{\Lambda_p}{k} \rightarrow 0,
\end{align*} 
w.h.p. as $k$ increases.
By the definition of $\Lambda_p$ we have: 
\begin{align*}
\phi_2\left(\Lambda_p\right) \leq \phi_2\left(0\right).
\end{align*} 
Recall the definition of function $\Gamma$ from (\ref{eq:Gamma_function}).
From Theorem~\ref{theorem:MainResult1} we have that w.h.p. as $k$ increases that
$\phi_2\left(\Lambda_p\right) \geq e^{-{3\over 2}} \Gamma\left(\frac{\Lambda_p}{k}\right).$ Combining the above two inequalities we derive that w.h.p.:
\begin{equation}
e^{-{3\over 2}}\Gamma\left(\frac{\Lambda_p}{k}\right) \leq \phi_2\left(0\right).
\label{eq:s111}
\end{equation}

Now from  $Y=X\beta^*+W$ we have 
$$
\phi_2\left(0\right)=n^{-\frac{1}{2}}||Y-X\beta^*||_2=n^{-\frac{1}{2}}||W||_2.
$$ 
Hence, $$\frac{1}{\sigma^2}\phi_2^{2}\left(0\right)=\frac{1}{\sigma^2}n^{-1}||W||^2_2=\frac{1}{n}\sum_{i=1}^{n} \left(\frac{W_i}{\sigma}\right)^2,$$ where $W_i$ are i.i.d. $N\left(0,\sigma^2\right)$. But by the Law of Large Numbers, w.h.p. 
$\frac{1}{\sigma^2}\phi_2^{2}\left(0\right)
=\frac{1}{n}\sum_{i=1}^{n} \left(\frac{W_i}{\sigma}\right)^2$ is less than $4  \E{\left(\frac{W_i}{\sigma}\right)^2}=4$. Hence, since $\Gamma\left(0\right)=\sigma$, this means that w.h.p. as $k$ (and therefore $n$) increases it holds:
$$\phi_2\left(0\right) \leq 2\sigma=2\Gamma\left(0\right).$$
Combining this with (\ref{eq:s111}) we get that w.h.p. as $k$ increases 
$$e^{-{3\over 2}}\Gamma\left(\frac{\Lambda_p}{k}\right) \leq 2 \sigma, $$
or equivalently 
$$e^{-{3\over 2}} \sqrt{2 \Lambda_p+\sigma^2} e^{-\frac{\Lambda_p \log p}{n}} \leq 2\sigma,$$ 
which we rewrite as
$$ e^{-{3\over 2}}\sqrt{\frac{2\Lambda_p}{\sigma^2}+1} \leq 2e^{\frac{\Lambda_p \log p}{n}}.$$
Now applying $n>\left(1+\epsilon\right)n^*$, we obtain, 
\begin{align*}
2e^{\frac{\Lambda_p\log p}{n}}<2e^{\frac{\Lambda_p\log p}{n^*\left(1+\epsilon\right)}}=2\left(\frac{2k}{\sigma^2}+1\right)^{\frac{\Lambda_p}{2\left(1+\epsilon\right)k}}.
\end{align*}  
But  $\Lambda_p \leq k$, and therefore 
\begin{align*}
2\left(\frac{2k}{\sigma^2}+1\right)^{\frac{\Lambda_p}{2\left(1+\epsilon\right)k}} 
\leq 2\left(\frac{2k}{\sigma^2}+1\right)^{\frac{1}{2\left(1+\epsilon\right)}}.
\end{align*} 
Combining  we obtain that w.h.p. as $k$ increases, 
$$
e^{-{3\over 2}}\sqrt{\frac{2\Lambda_p}{\sigma^2}+1} \leq 2\left(\frac{2k}{\sigma^2}+1\right)^{\frac{1}{2\left(1+\epsilon\right)}},
$$ 
which after squaring and rearranging gives w.h.p.,
$$
\frac{2\Lambda_p}{\sigma^2} \leq 4e^3 \left(\frac{2k}{\sigma^2}+1\right)^{\frac{1}{\left(1+\epsilon\right)}}-1,
$$ 
which we further rewrite as
\begin{equation}
\frac{\Lambda_p}{k} \leq \frac{\sigma^2}{2k} \left(4e^3 \left(\frac{2k}{\sigma^2}+1\right)^{\frac{1}{\left(1+\epsilon\right)}}-1 \right).
\label{eq:2}
\end{equation}
We claim that this upper bound   tends to zero, as $k \rightarrow +\infty$. 
Indeed, let $x_k=\frac{k}{\sigma^2}$. By the assumption of the theorem 
$x_k \rightarrow + \infty$. But the right-hand side of (\ref{eq:2}) can be upper bounded by a constant multiple of 
$x_k^{-1} x_k^{\frac{1}{1+\epsilon}}=x_k^{-\frac{\epsilon}{1+\epsilon}}$, which converges to zero as $k$ increases. 
Therefore from (\ref{eq:2}), $\frac{\Lambda_p}{k} \rightarrow 0$ w.h.p. as $k$ increases, and the proof is complete in that case.

\vspace{.1in}
\textbf{ Case 2}: $\frac{1}{C} k \log k <n<\left(1-\epsilon\right)n^*$.
First we check that this regime for $n$ is well-defined. Indeed the assumption $\max \{k,\frac{2k}{\sigma^2}+1\} \leq \exp \left(  \sqrt{ C \log p } \right)$ implies that  it holds
\begin{align}\label{eq:boundonk}
n^*=\frac{2k \log p}{ \log \left(\frac{2k}{\sigma^2}+1 \right) }\geq \frac{2k \log p}{\sqrt{C \log p}} \geq \frac{2}{C} k \log k>\frac{1}{C}k \log k.
\end{align}

Now we need to show that w.h.p. as $k$ increases 
$$
\frac{\Lambda_p}{k} \rightarrow 1.
$$

By the definition of $\Lambda_p$, $\phi_2\left(\Lambda_p\right)\leq \phi_2\left(1\right)$. 
Again applying  Theorem ~\ref{theorem:MainResult1} we have that w.h.p. as $k$ increases it holds 
$\phi_2\left(\Lambda_p\right) \geq e^{-{3\over 2}} \Gamma\left(\frac{\Lambda_p}{k}\right).$ Combining the above two inequalities we obtain that w.h.p.,
\begin{equation}
e^{-{3\over 2}}\Gamma\left(\frac{\Lambda_p}{k}\right) \leq \phi_2\left(1\right).
\label{eq:s23}
\end{equation} 
Now we apply  the second part of Theorem~\ref{theorem:MainResult1}. Given any $D_0$ from part (b) of Theorem~\ref{theorem:MainResult1}
and since $k/\sigma\rightarrow\infty$,
we have that $\frac{1}{C}k \log k \le n\le (1-\epsilon)n^*$ furthermore then satisfies $\frac{1}{C}k \log k \le n\le k\log p/(3\log D_0)$ for all sufficiently large $k$. 
We obtain that w.h.p. as $k$ increases
$$
\phi_2\left(1\right) \leq  D_0 \Gamma\left(1\right).
$$ 
Using this in (\ref{eq:s23}) and letting $c=1/(e^{3\over 2}D_0)$
we obtain 
$$
c\Gamma\left( \frac{\Lambda_{p}}{k}\right) \leq \Gamma\left(1\right),
$$ 
namely,
$$
c\sqrt{\frac{2 \Lambda_p}{\sigma^2}+1}e^{-\frac{\Lambda_{p} \log p}{n}}\le \sqrt{\frac{2 k}{\sigma^2}+1}e^{-\frac{k\log p}{n}},
$$ 
and therefore
\begin{equation}
c^2\left(\frac{2\Lambda_p+\sigma^2}{2k+\sigma^2}\right) =
c^2\left(\frac{\frac{2\Lambda_p}{\sigma^2}+1}{\frac{2k}{\sigma^2}+1}\right)\leq e^{\frac{2\left(\Lambda_p-k\right)\log p}{n}}.
\label{eq:44}
\end{equation} 
Now using $n\le\left(1-\epsilon\right)n^*$ and $\Lambda_p-k \leq 0$, we obtain 
$$
e^{\frac{2\left(\Lambda_p-k\right)\log p}{n}} 
\leq e^{\frac{2\left(\Lambda_p-k\right)\log p}{\left(1-\epsilon\right)n^*}}=
\left(\frac{2k}{\sigma^2}+1\right)^{-\frac{k-\Lambda_p}{k\left(1-\epsilon\right)}}.
$$ 
Combining the above with (\ref{eq:44}) we obtain that w.h.p., 
\begin{align*} 
c^2\left(\frac{2\Lambda_p+\sigma^2}{2k+\sigma^2}\right) \leq 
\left(\frac{2k}{\sigma^2}+1\right)^{-\frac{k-\Lambda_p}{k\left(1-\epsilon\right)}},
\end{align*} 
or w.h.p., 
\begin{equation}
c^2\left(\frac{2\Lambda_p}{\sigma^2}+1\right) \leq 
\left(\frac{2k}{\sigma^2}+1\right)^{-\frac{\epsilon}{1-\epsilon}+\frac{\Lambda_p}{k\left(1-\epsilon\right)}}.
\label{eq:s24}
\end{equation}
from which we obtain a simpler bound
\begin{align*}
&c^2 \leq \left(\frac{2k}{\sigma^2}+1\right)^{-\frac{\epsilon}{1-\epsilon}+\frac{\Lambda_p}{k\left(1-\epsilon\right)}},
\end{align*}
namely
\begin{align*}
2\log c \leq  \left(-\frac{\epsilon}{1-\epsilon}+\frac{\Lambda_p}{k\left(1-\epsilon\right)}\right) \log \left( \frac{2k}{\sigma^2}+1 \right)
\end{align*}
or
\begin{align*}
\frac{2 \log c}{\log \left( \frac{2k}{\sigma^2}+1 \right)} \left(1-\epsilon\right)+ \epsilon \leq \frac{\Lambda_p}{k}.
\end{align*}
Since by the assumption of the theorem we have  $k/\sigma^{2}\rightarrow\infty$, we obtain that 
$\frac{\Lambda_p}{k}\ge \epsilon/2$ w.h.p. as $k\rightarrow\infty$.
Now we reapply this bound for (\ref{eq:s24}) and obtain that w.h.p.
\begin{align*}
c^2\left(\frac{\epsilon k}{\sigma^2}+1\right) \leq 
\left(\frac{2k}{\sigma^2}+1\right)^{-\frac{\epsilon}{1-\epsilon}+\frac{\Lambda_p}{k\left(1-\epsilon\right)}}.
\end{align*}
Taking logarithm of both sides, we obtain that w.h.p.
\begin{align*}
(1-\epsilon)\log^{-1}\left({2k\over \sigma^{2}}+1\right)\left(\log\left({\epsilon k\over \sigma^{2}}+1\right)+2\log c \right)
+\epsilon\le {\Lambda_{p} \over k}.
\end{align*}
Now again since $k/\sigma^{2}\rightarrow\infty$, it is easy to see that the ratio of two logarithms approaches unity as $k$ increases, 
and thus the limit of the left-hand side is $1-\epsilon+\epsilon=1$ in the limit. Thus $\Lambda_{p}/k$ approaches unity in the limit w.h.p.
as $k$ increases. This completes the proof.
\end{proof}

\section{The Overlap Gap Property}\label{section:OGP}
In this section we prove Theorem~\ref{theorem:OGP}.
We begin by establishing a certain property regarding the  the limiting curve function $\Gamma$. 
\begin{lemma}\label{lemma:BoundGamma}
Under the assumption of Theorem~\ref{theorem:OGP}, there exist sequences
$0<\zeta_{1,k,n}<\zeta_{2,k,n}<1$ such that $\lim_k k\left(\zeta_{2,k,n}-\zeta_{1,k,n}\right)=+ \infty$ 
 and such that for all sufficiently large $k$ 
\begin{align*}
\inf_{\zeta \in \left(\zeta_{1,k,n},\zeta_{2,k,n}\right)} 
\min\left( \frac{\Gamma\left(\zeta\right)}{\Gamma\left(0\right)},\frac{\Gamma\left(\zeta\right)}{\Gamma\left(1\right)} \right) \geq e^3D_0.
\end{align*}
\end{lemma}

\begin{proof}
Recall that 
$\Gamma\left(0\right)=\sigma$ and $\Gamma\left(1\right)=\sqrt{2k+\sigma^2}\exp\left( -\frac{k \log p}{n}\right)$. 
We will rely on the results of Proposition~\ref{prop:GammaMonotonic} and thus recall the definition of $n^{*}$.

Assume now $n^* \le n< \frac{k \log p}{3\log D_0}$. We choose $\zeta_{1,k,n}=\frac{1}{5}$ and $\zeta_{2,k,n}=\frac{1}{4}$. 
Clearly $k\left(\zeta_{2,k,n}-\zeta_{1,k,n}\right) \rightarrow + \infty$. 
Since $n\ge n^{*}$ we know that $\Gamma\left(0\right)<\Gamma\left(1\right)$ and therefore it suffices to show  
\begin{align*}
\inf_{\zeta \in \left(\zeta_{1,k,n},\zeta_{2,k,n}\right)} 
\frac{\Gamma\left(\zeta\right)}{\Gamma\left(1\right)} \geq e^3D_0.
\end{align*}Using the log-concavitiy of $\Gamma$ and squaring both side it suffices to establish
\begin{align*}
\min \left( \left(\frac{\Gamma\left(\zeta_{1,k,n}\right)}{\Gamma\left(1\right)}\right)^2,
\left(\frac{\Gamma\left(\zeta_{2,k,n}\right)}{\Gamma\left(1\right)}\right)^2 \right) > e^6D_0^2.
\end{align*} 
But since $n< k \log p/(3\log D_0)$ have 
\begin{align*}
\min \left( \left(\frac{\Gamma\left(\zeta_{1,k,n}\right)}{\Gamma\left(1\right)}\right)^2,
\left(\frac{\Gamma\left(\zeta_{2,k,n}\right)}{\Gamma\left(1\right)}\right)^2 \right)
& = \min \left( \frac{\frac{2k}{5}+\sigma^2}{2k+\sigma^2}e^{4k \log p\over 5n},
\frac{\frac{3k}{4}+\sigma^2}{2k+\sigma^2}e^{3k \log p\over 4n}  \right)\\
&\ge \min \left( {1\over 4}D_0^{12\over 5},{2\over 3}D_0^{9\over 4}  \right)\\
&>e^6D_0^2,
\end{align*}  
for all sufficiently large $D_0$.
This completes the proof of the lemma.
\end{proof}

Now we return to the proof of  Theorem~\ref{theorem:OGP}. 
\begin{proof}[Proof of Theorem~\ref{theorem:OGP}]
Choose $0<\zeta'_{1,k,n}<\zeta'_{2,k,n}<1$  from Lemma~\ref{lemma:BoundGamma} 
and set $r_k=D_0\max \left( \Gamma\left(0\right), \Gamma\left(1\right) \right)$. We will now prove that for this value of $r_k$ 
and $\zeta_{1,k,n}=1-\zeta'_{2,k,n},\zeta_{2,k,n}=1-\zeta'_{1,k,n}$, the set $S_{r_k}$ satisfies the claim of the theorem. 
Applying the second part of Theorem~\ref{theorem:MainResult1} we obtain $\beta^* \in S_{r_{k}}$ since 
$n^{-\frac{1}{2}}||Y-X\beta^*||_2=n^{-{1\over 2}}\sqrt{\sum_i W_i^2}$ which by the Law of Large Numbers is w.h.p.
at most $2\sigma=2\Gamma(0)<r_k$, provided $D_0$ is sufficiently large.
This establishes (b).
We also note that (c) follows immediately from Theorem~\ref{theorem:MainResult1}.

We now establish part (a). 
Assume there exists a $\beta\in S_{r_{k}}$ with overlap $\zeta \in \left(\zeta_{1,k,n},\zeta_{2,k,n}\right)$. 
This implies that the optimal value of the optimization problem $\Phi_{2}(\ell)$ satisfies 
\begin{equation}
\phi_2\left(k\left(1-\zeta\right)\right)\le r_k.
\label{eq:45}
\end{equation} 
Now  $1-\zeta \in \left(1-\zeta_{2,k,n},1-\zeta_{1,k,n}\right)=\left(\zeta'_{1,k,n},\zeta'_{2,k,n}\right)$ and Lemma~\ref{lemma:BoundGamma}
imply
$$
e^3D_0 \max \{ \Gamma\left(0\right), \Gamma\left(1\right) \}\le\Gamma\left(1-\zeta\right).
$$ 
We obtain
\begin{align*}
r_k\le e^{-3}\Gamma\left(1-\zeta\right),
\end{align*}
which combined with (\ref{eq:45}) contradicts the first part of Theorem~\ref{theorem:MainResult1}.
\end{proof}

\section{Conclusions and Open Questions}
Our paper prompts several new directions for research. Relaxing the assumption that regression coefficients are binary
is a natural first step in extending the results of this paper. We believe that both the general picture and the main approach should remain the same 
in this setting, where 
appropriate discretization of the  coefficient of the regression vector values might be a viable approach. 
Furthermore, it would be interesting to see
if the conditional second moment approach proposed in this paper can be used to obtain squared error associated with  the
relaxation of the problem  such as LASSO and the Compressive Sensing methods.

An interesting question is to see as to what extent $n^*$ is indeed the information theoretic limit for the problem
of recovery of $\beta^*$ in a strong sense. As per the results of~\cite{wang2010information}, the application of the Gaussian channel
estimates imply that below this threshold the precise recovery of $\beta^*$ is impossible information theoretically.
However, it is not ruled out that it might be possible to recover at least a portion of the support of $\beta^*$.
Our results show that the method based on minimizing the squared error is a poor help for this problem as the optimal solution $\beta_2$
misses the support almost completely. But it is not ruled out that some other method is capable of recovering at least some positive
fraction of the support of $\beta^*$. We conjecture that this is not the case and that below $n^*$ the recovery
of $\beta^*$ is impossible in the very strong sense that even obtaining  a fraction of support $\beta^*$ is not possible information
theoretically. Similarly, motivated by the fact that $n_{\rm inf,1}$ is asymptotic information theoretic limit when $k=1$ and $\sigma$ grows,
it would be interesting to see if the recovery of any part of the support of $\beta^*$ is possible when $n<n_{\rm inf,1}$. We conjecture that with $n<n_{\mathrm{inf},1}$ samples, the recovery of even one coordinate of the support of $\beta^*$ is impossible. 

We pose also a geometric question on the importance of $n_{\rm inf,1}$. Notice that for $n<n_{\mathrm{inf},1}$ the first moment curve $\Gamma(\cdot)$ is decreasing, while for $n>n_{\mathrm{inf},1}$ the first moment curve $\Gamma(\cdot)$ first increases and then decreases (in particular, it is non-monotonic). Assuming the $\Gamma(\cdot)$ is a tight approximation of $\Phi_2(\cdot)$ lead to the statement and proof of both Theorems \ref{theorem:OGP} and \ref{theorem:sharptheorem} in the main body paper. Making this assumption now for the behavior around $n_{\rm inf, 1}$ it suggests the following implication for the local geometry of the solution space of $(\Phi_2)$ around $\beta^*$: for $n<n_{\mathrm{inf},1}$ the ground truth $\beta^*$ is \textbf{not} a Hamming-distance local minimum in $(\Phi_2)$ (as $\Phi_2(\cdot)$ is decreasing at $\ell=0$) , while for $n>n_{\mathrm{inf},1}$  the ground truth $\beta^*$ is a Hamming-distance local minimum in $(\Phi_2)$ (as $\Phi_2(\cdot)$ is increasing at $\ell=0$.) Notice that this local property does not imply that recovery of $\beta^*$ is possible: there are potentially many other local minima in $\Phi_2$ (something actually true as $n_{\mathrm{inf},1}<n^*$ and $n^*$ is the proven information-theoretic limit.) We pose the establishment of this property an an interesting information-theoretic open problem.

Our results apply to the case when the sampling size $n$ is essentially of the order $o(k\log p)$ (though a small constant in front 
of $k\log p$ is allowed). Obtaining
estimates of the squared error for the regime between $o(k\log p)$ and the LASSO/Compressive Sensing 
threshold $n_{\text{LASSO/CS}}=(2k+\sigma^2)\log p$ is of interest.
This appears to be a difficult regime, as in this case the gap between the conditional first and second moment widens as $n$ approaches
the order $O(k\log p)$. It is possible that non-rigorous methods of Replica Symmetry Breaking might be of help here to obtain
at least good predictions for the answers. Such predictions are available in the regime when $k,n$ and $p$ are of the same
order~\cite{bayati2011dynamics},\cite{zheng2015does}.

Last but not the least, understanding the algorithmic complexity of the problem of finding $\beta^*$ when $n$ is between
$n^*$ and $n_{\text{LASSO/CS}}$ is of great interest. The Overlap Gap Property established in this paper suggests that
the problem might indeed be algorithmically hard, though such formal hardness results are lacking even for  random constraint
satisfaction problems for which the OGP was known already for a long time. On the other hand, it is often observed
that for random constraint satisfaction problems outside the regime where OGP takes place, even very naive algorithms
such as greedy type algorithms are successful. By drawing an analogy
between this class of problems and the problems of high dimensional regression, it is possible that above say  threshold $n_{\text{LASSO/CS}}$
some version of a greedy algorithm is successful in recovering the regression vector $\beta^*$. Similarly, it would interesting
to establish that the OGP ceases to exist above the threshold $n_{\rm LASSO/CS}$.

\section*{Aknowledgment}
The authors would like to thank Philippe Rigollet for helpful discussions during the preparation of this paper.


\bibliographystyle{amsalpha}


\begin{thebibliography}{ACORT11}

\bibitem[ACO08]{achlioptas2008algorithmic}
Dimitris Achlioptas and Amin Coja-Oghlan, \emph{Algorithmic barriers from phase
  transitions}, Foundations of Computer Science, 2008. FOCS'08 IEEE 49th Annual
  IEEE Symposium on, IEEE, 2008, pp.~793--802.

\bibitem[ACORT11]{AchlioptasCojaOghlanRicciTersenghi}
D.~Achlioptas, A.~Coja-Oghlan, and F.~Ricci-Tersenghi, \emph{On the solution
  space geometry of random formulas}, Random Structures and Algorithms
  \textbf{38} (2011), 251--268.

\bibitem[BB99]{Brunel99}
L.~Brunel and J.~Boutros, \emph{Euclidean space lattice decoding for joint
  detection in cdma systems.}, IEEE Information Theory and Communications
  Workshop (1999).

\bibitem[BBHL09]{BickelGenome}
Peter~J. Bickel, James~B. Brown, Haiyan Huang, and Qunhua Li, \emph{An overview
  of recent developments in genomics and associated statistical methods}, Phil.
  Trans. R. Soc. A (2009).

\bibitem[BBSV18]{Bollobas18}
Paul Balister, Bela Bollobas, Julian Sahasrabudhe, and Alexander Veremyev,
  \emph{Dense subgraphs in random graphs}, arXiv (2018).

\bibitem[BLH{\etalchar{+}}14]{MRImed}
Zhao Bo, Wenmiao Lu, T.~Kevin Hitchens, Fan Lam, Chien Ho, and Zhi‐Pei Liang,
  \emph{Accelerated mr parameter mapping with low‐rank and sparsity
  constraints}, Magnetic Resonance in Medicine (2014).

\bibitem[BM11]{bayati2011dynamics}
Mohsen Bayati and Andrea Montanari, \emph{The dynamics of message passing on
  dense graphs, with applications to compressed sensing}, IEEE Transactions on
  Information Theory \textbf{57} (2011), no.~2, 764--785.

\bibitem[BMV{\etalchar{+}}17]{Jiaming17}
Jess Banks, Christopher Moore, Nicolas Verzelen, Roman Vershynin, and Jiaming
  Xu, \emph{Information-theoretic bounds and phase transitions in clustering,
  sparse pca, and submatrix localization}, IEEE Trans. Inf. Theory, vol. 67,
  no. 7, pp. 4872-4894 (July 2017).

\bibitem[BP17]{BertsimasRegression}
D.~Bertsimas and Bart~Van Parys, \emph{Sparse high dimensional regression:
  Exact scalable algorithms and phase transitions}, arXiv preprint
  arXiv:1709.10029 (2017).

\bibitem[CCL{\etalchar{+}}08]{JASAgenomics}
Carlos~M. Carvalho, Jeffrey Chang, Joseph~E. Lucas, Joseph~R. Nevins, Quanli
  Wang, and Mike West, \emph{High-dimensional sparse factor modeling:
  Applications in gene expression genomics}, Journal of the American
  Statistical Association (2008).

\bibitem[CG18]{Tony18}
T.~Tony Cai and Zijian Gao, \emph{Accuracy assessment for high-dimensional
  linear regression1}, The Annals of Statistics (2018).

\bibitem[CL99]{Chen99}
V.C. Chen and Hao Ling, \emph{Joint time-frequency analysis for radar signal
  and image processing}, IEEE Transactions on Signal Processing (1999).

\bibitem[COE11]{coja2011independent}
A.~Coja-Oghlan and C.~Efthymiou, \emph{On independent sets in random graphs},
  Proceedings of the Twenty-Second Annual ACM-SIAM Symposium on Discrete
  Algorithms, SIAM, 2011, pp.~136--144.

\bibitem[CT05]{candes2005decoding}
Emmanuel~J Candes and Terence Tao, \emph{Decoding by linear programming}, IEEE
  transactions on information theory \textbf{51} (2005), no.~12, 4203--4215.

\bibitem[CT06]{Cover}
Thomas~M. Cover and Joy~A. Thomas, \emph{Elements of information theory (wiley
  series in telecommunications and signal processing)}, Wiley-Interscience,
  2006.

\bibitem[Don06]{donoho2006compressed}
David~L Donoho, \emph{Compressed sensing}, IEEE Transactions on information
  theory \textbf{52} (2006), no.~4, 1289--1306.

\bibitem[DT10]{donoho2006counting}
David~L. Donoho and Jared Tanner, \emph{Counting the faces of
  randomly-projected hypercubes and orthants, with applications}, Discrete {\&}
  Computational Geometry \textbf{43} (2010), no.~3, 522--541.

\bibitem[Dud17]{Du17}
J.~Dudczyk, \emph{A method of feature selection in the aspect of specific
  identification of radar signals}, Bulletin of the Polish Academic of
  Sciences, Technical Sciences (2017).

\bibitem[EACP11]{Casto11}
Emmanuel J.~Cand`es Ery Arias-Castro and Yaniv Plan, \emph{Global testing under
  sparse alternatives: Anova, multiple comparisons and the higher criticism},
  The Annals of Statistics (2011).

\bibitem[FL10]{Fan10}
Jianqing Fan and Jinchi Lv, \emph{A selective overview of variable selection in
  high dimensional feature space}, Statistica Sinica (2010).

\bibitem[FR13]{foucart2013mathematical}
Simon Foucart and Holger Rauhut, \emph{A mathematical introduction to
  compressive sensing}, Springer, 2013.

\bibitem[Geo12]{Ed2012}
Edward~I. George, \emph{The variable selection problem}, Journal of the
  American Statistical Association (2012).

\bibitem[GL16]{gamarnik2016finding}
David Gamarnik and Quan Li, \emph{Finding a large submatrix of a gaussian
  random matrix}, arXiv preprint arXiv:1602.08529 (2016).

\bibitem[GSa]{gamarnik2014limits}
David Gamarnik and Madhu Sudan, \emph{Limits of local algorithms over sparse
  random graphs}, Annals of Probability. {\rm To appear}.

\bibitem[GSb]{gamarnik2014performance}
\bysame, \emph{Performance of sequential local algorithms for the random
  nae-k-sat problem}, SIAM Journal on Computing. {\rm To appear}.

\bibitem[GZ18]{GZ18}
David Gamarnik and Ilias Zadik, \emph{High dimensional linear regression using
  lattice basis reduction}, Neural Information Processing Systems (NeurIPS),
  2018.

\bibitem[HB98]{Hassibi98}
A.~Hassibi and S.~Boyd, \emph{Integer parameter estimation in linear models
  with applications to gps}, IEEE Transactions on Signal Processing (1998).

\bibitem[HC08]{Ho08}
Stefani M. dos Remedios C.~G. Ho, J.~W. and M.~A. Charleston,
  \emph{Differential variability analysis of gene expression and its
  application to human diseases}, Bioinformatics (2008).

\bibitem[HG10]{Hu10}
Qiu~X. Hu, R. and G.~Glazko, \emph{A new gene selection procedure based on the
  covariance distance}, Bioinformatics (2010).

\bibitem[HTW15]{Book15}
Trevor Hastie, Robert Tibshirani, and Martin~J. Wainwright, \emph{Statistical
  learning with sparsity: The lasso and generalizations}, Chapman and Hall/CRC
  Monographs on Statistics and Applied Probability, 2015.

\bibitem[HV02]{Hassibi02}
B.~Hassibi and H.~Vikalo, \emph{On the expected complexity of integer
  least-squares problems}, IEEE International Conference on Acoustics, Speech,
  and Signal Processing. (2002).

\bibitem[HY09]{Hu09}
Qiu X. Glazko G. Klevanov~L. Hu, R. and A.~Yakovlev, \emph{Detecting intergene
  correlation changes in microarray analysis: A new approach to gene
  selection}, BMC Bioinformatics (2009).

\bibitem[JBC17]{Lucas17}
Lucas Janson, Rina~Foygel Barber, and Emmanuel Cand{\`e}s, \emph{Eigenprism:
  inference for high dimensional signal-to-noise ratios}, Journal of the Royal
  Statistical Society. Series B (2017).

\bibitem[LDSP08]{DonohoMRI}
M.~Lustig, D.~L. Donoho, J.~M. Santos, and J.~M. Pauly, \emph{Compressed
  sensing mri}, IEEE Signal Processing Magazine \textbf{25} (2008), no.~2,
  72--82.

\bibitem[MB06a]{Aos06}
Nicolai Meinshausen and Peter B{\"u}hlmann, \emph{High-dimensional graphs and
  variable selection with the lasso}, The Annals of Statistics (2006).

\bibitem[MB06b]{graph}
Nicolai Meinshausen and Peter B{\"u}hlmann, \emph{High-dimensional graphs and
  variable selection with the lasso}, Ann. Statist. \textbf{34} (2006), no.~3,
  1436--1462.

\bibitem[MRT11]{montanari2011reconstruction}
Andrea Montanari, Ricardo Restrepo, and Prasad Tetali, \emph{Reconstruction and
  clustering in random constraint satisfaction problems}, SIAM Journal on
  Discrete Mathematics \textbf{25} (2011), no.~2, 771--808.

\bibitem[NG13]{Van13}
Richard Nickl and Sara Van~De Geer, \emph{Confidence sets in sparse
  regression}, The Annals of Statistics (2013).

\bibitem[OWJ11]{Martin11}
Guillaumem Obozinski, Martin~J. Wainwright, and Michael~I. Jordan,
  \emph{Support union recovery in high-dimensional multivariate regression},
  The Annals of Statistics (2011).

\bibitem[PZHS16]{CSsensor}
Bao Peng, Zhi Zhao, Guangjie Han, and Jian Shen, \emph{Consensus-based sparse
  signal reconstruction algorithm for wireless sensor networks}, International
  Journal of Distributed Sensor Networks (2016).

\bibitem[QMP{\etalchar{+}}12]{CSwireless}
Giorgio Quer, Riccardo Masiero, Gianluigi Pillonetto, Michele Rossi, and
  Michele Zorzi, \emph{Sensing, compression, and recovery for wsns: Sparse
  signal modeling and monitoring framework}, IEEE Transactions on Wireless
  Communications (2012).

\bibitem[RG12]{Galen12}
G.~Reeves and M.~Gastpar, \emph{The sampling rate-distortion tradeoff for
  sparsity pattern recovery in compressed sensing}, IEEE Transactions on
  Information Theory \textbf{58} (2012), no.~10, 3065--3092.

\bibitem[RG13a]{Reeves13}
\bysame, \emph{Approximate sparsity pattern recovery: Information-theoretic
  lower bounds}, IEEE Transactions on Information Theory \textbf{59} (2013),
  no.~6, 3451--3465`.

\bibitem[RG13b]{Galen13}
Galen Reeves and Michael Gapstar, \emph{Approximate sparsity pattern recovery:
  Information-theoretic lower bounds}, IEEE Trans. Information Theory (2013).

\bibitem[RV14]{rahman2014local}
Mustazee Rahman and Balint Virag, \emph{Local algorithms for independent sets
  are half-optimal}, arXiv preprint arXiv:1402.0485 (2014).

\bibitem[SC15]{Scarlett15}
Jonathan Scarlett and Volkan Cevher, \emph{Limits on support recovery with
  probabilistic models: An information-theoretic framework}, IEEE International
  Symposium on Information Theory (ISIT) (2015).

\bibitem[SW95]{large_deviations_ShWeiss}
A.~Shwartz and A.~Weiss, \emph{Large deviations for performance analysis},
  Chapman and Hall, 1995.

\bibitem[TWY12]{Cai2012}
Cai Tony, Liu Weidong, and Xia Yin, \emph{Two-sample covariance matrix testing
  and support recovery in high-dimensional and sparse settings}, Journal of the
  American Statistical Association (2012).

\bibitem[TZP19]{Thr19}
Christos Thrampoulidis, Ilias Zadik, and Yury Polyanskyi, \emph{A simple bound
  on the ber of the map decoder for massive mimo systems}, IEEE International
  Conference on Acoustics, Speech, and Signal Processing. (2019).

\bibitem[Wai09a]{wainwright2009information}
Martin~J Wainwright, \emph{Information-theoretic limits on sparsity recovery in
  the high-dimensional and noisy setting}, Information Theory, IEEE
  Transactions on \textbf{55} (2009), no.~12, 5728--5741.

\bibitem[Wai09b]{wainwright2009sharp}
\bysame, \emph{Sharp thresholds for high-dimensional and noisy sparsity
  recovery using constrained quadratic programming (lasso)}, IEEE transactions
  on information theory \textbf{55} (2009), no.~5, 2183--2202.

\bibitem[WWR10]{wang2010information}
Wei Wang, Martin~J Wainwright, and Kannan Ramchandran,
  \emph{Information-theoretic limits on sparse signal recovery: Dense versus
  sparse measurement matrices}, IEEE Transactions on Information Theory
  \textbf{56} (2010), no.~6, 2967--2979.

\bibitem[XZB01]{ZH01}
Y.~Shi X.~Zhang and Z.~Bao, \emph{A new feature vector using selected bispectra
  for signal classification with application in radar target recognition}, IEEE
  Transactions on Signal Processing (2001).

\bibitem[Zha93]{Model93}
Ping Zhang, \emph{Model selection via multifold cross validation}, The Annals
  of Statistics (1993).

\bibitem[ZMWL15]{zheng2015does}
Le~Zheng, Arian Maleki, Xiaodong Wang, and Teng Long, \emph{Does $\backslash
  ell\_p $-minimization outperform $\backslash ell\_1 $-minimization?}, arXiv
  preprint arXiv:1501.03704 (2015).

\bibitem[ZTP19]{Zad19}
Ilias Zadik, Christos Thrampoulidis, and Yury Polyanskyi, \emph{Improved bounds
  on gaussian mac and sparse regression via gaussian inequalities}, IEEE
  International Symposium on Information Theory (ISIT) (2019).

\bibitem[ZY06]{Zhao}
Peng Zhao and Bin Yu, \emph{On model selection consistency of lasso}, J. Mach.
  Learn. Res. \textbf{7} (2006), 2541--2563.

\end{thebibliography}

\newcommand{\etalchar}[1]{$^{#1}$}
\providecommand{\bysame}{\leavevmode\hbox to3em{\hrulefill}\thinspace}
\providecommand{\MR}{\relax\ifhmode\unskip\space\fi MR }
\providecommand{\MRhref}[2]{%
  \href{http://www.ams.org/mathscinet-getitem?mr=#1}{#2}
}
\providecommand{\href}[2]{#2}

\end{document}